\documentclass[10pt,twocolumn,twoside]{IEEEtran}

\usepackage[sort,compress]{cite}
\usepackage{amsmath,amsfonts,amsthm}
\usepackage{dsfont}
\usepackage{algorithm,algorithmic}
\usepackage{subcaption}
\usepackage{mathtools}
\usepackage{graphicx}
\usepackage{textcomp}
\usepackage{cite}
\usepackage{amssymb}
\usepackage{hyperref}
\usepackage{booktabs}
\usepackage{threeparttable}

\newtheorem{remark}{Remark}
\newtheorem{result}{Main Result}
\newtheorem{assumption}{Assumption}
\newtheorem{theorem}{Theorem}
\newtheorem{definition}{Definition}
\newtheorem{lemma}{Lemma}

\newcommand\IND{\mathds{1}}
\newcommand\PR{\mathds{P}}
\newcommand\EXP{\mathds{E}}

\newcommand\reals{\mathbb{R}}

\newcommand{\bs}{\boldsymbol}
\newcommand*\REGRET{\mathcal{R}}
\newcommand*\tildeO{\tilde{\mathcal{O}}}
\newcommand*\OO{\mathcal{O}}
\newcommand*\Sqrt[1]{\sqrt{\smash{#1}\vphantom{\bar S}}}
\DeclareMathOperator\SPAN{span}
\DeclareMathOperator\DIAM{diam}
\DeclareMathOperator*{\argmin}{arg\,min}

\newcommand*\HIDE[1]{}

\newcommand*\PP{p}

\begin{document}
	\title{On learning Whittle index policy for restless bandits with scalable regret}
	
	\author{Nima Akbarzadeh, Aditya Mahajan
		\thanks{This research was funded in part by the Innovation for Defence Excellence and Security (IDEaS) Program of the Canadian Department of National Defence through grant CFPMN2-037, and Fonds de Recherche du Quebec-Nature et technologies (FRQNT).}
		\thanks{Nima Akbarzadeh and Aditya Mahajan are with Department of Electrical and Computer Engineering, McGill University, Montreal. (\href{mailto:nima.akbarzadeh@mail.mcgill.ca}{nima.akbarzadeh@mail.mcgill.ca}, \href{mailto:aditya.mahajan@mcgill.ca}{aditya.mahajan@mcgill.ca})}}
	
	\maketitle
	
	\begin{abstract}
		Reinforcement learning is an attractive approach to learn good resource allocation and scheduling policies based on data when the system model is unknown. However, the cumulative regret of most RL algorithms scales as $\tilde O(\mathsf{S} \sqrt{\mathsf{A} T})$, where $\mathsf{S}$ is the size of the state space, $\mathsf{A}$ is the size of the action space, $T$ is the horizon, and the $\tilde{O}(\cdot)$ notation hides logarithmic terms. Due to the linear dependence on the size of the state space, these regret bounds are prohibitively large for resource allocation and scheduling problems. In this paper, we present a model-based RL algorithm for such problems which has scalable regret. In particular, we consider a restless bandit model, and propose a Thompson-sampling based learning algorithm which is tuned to the underlying structure of the model. We present two characterizations of the regret of the proposed algorithm with respect to the Whittle index policy. First, we show that for a restless bandit with $n$ arms and at most $m$ activations at each time, the regret scales either as $\tilde{O}(mn\sqrt{T})$ or $\tilde{O}(n^2 \sqrt{T})$ depending on the reward model. Second, under an additional technical assumption, we show that the regret scales as $\tilde{O}(n^{1.5} \sqrt{T})$ or $\tilde{O}(\max\{m\sqrt{n}, n\} \sqrt{T})$. We present numerical examples to illustrate the salient features of the algorithm.
	\end{abstract}

%	\begin{IEEEkeywords}
%		Restless bandits, Thompson sampling, reinforcement learning, Whittle index
%	\end{IEEEkeywords}
	
	\section{Introduction}
	Resource allocation and scheduling problems arise in control of networked
	systems. Examples include 
	opportunistic scheduling in networks~\cite{ouyang2015downlink,borkar2017opportunistic,wang2019opportunistic},
	link scheduling in machine type communication~\cite{ali2018sleeping},
	user allocation in mmWave networks~\cite{singh2022user},
	channel allocation in networks~\cite{lott2000optimality},
	source selection in peer-to-peer networks~\cite{si2009distributed},
	opportunistic spectrum
	access~\cite{liu2010indexability,nino2009restless,tekin2011online},
	demand response in smart grids~\cite{wang2014adaptive,abad2016near},
	dynamic routing in multi-UAVs~\cite{le2008multi},
	operator allocation in multi-robot systems~\cite{dahiya2022scalable}, etc.
	
	Due to the curse of dimensionality, finding an optimal solution in such
	resource allocation and scheduling problems is computationally
	prohibitive~\cite{papadimitriou1999complexity}. Restless bandits
	(RBs)~\cite{whittle1988restless} have emerged as a popular solution heuristic
	for such problems. The RB framework is motivated by the \emph{rested} multi-armed bandit problem considered in the seminar work of Gittins~\cite{gittins1979bandit}, who showed that the optimal strategy for the rested multi-armed bandit problem is of
	the \emph{index type}: one can compute an index for each state of each alternative (also called an arm), and choose the alternative with the highest index. In general, such index-type policies are not optimal for RBs. In fact, computing the optimal policy for RBs is PSPACE hard~\cite{papadimitriou1999complexity}. However, as argued in~\cite{whittle1988restless}, an index-type policy (now known as the Whittle index) can be a useful heuristic if a technical condition known as indexability is satisfied. The Whittle index policy is optimal for some specific models~\cite{gittins1979bandit, weber1990index, lott2000optimality}.
	There is also a strong empirical evidence to suggest that the Whittle index policy performs close to optimal in various settings~\cite{glazebrook2002index,glazebrook2005index,glazebrook2006some,ninomora2007,ayesta2010modeling,akbarzadeh2022conditions}. For these reasons, the RB framework has been applied in a variety of resource allocation and scheduling applications referenced above.
	
	In all the above references, it is assumed that the system model is known
	perfectly. In many real-world applications, there is often uncertainty about
	the system model. In such situations, RL (reinforcement learning) is an attractive
	alternative. In recent years, there are many papers which investigate
	RL for RBs~\cite{meshram2017restless,borkar2018learning,Fu2019Qlearning,avrachenkov2020whittle,robledo2022qwi}. Most of these learn the Q-function associated with the average reward/cost optimality equation parameterized by the activation cost~$\lambda$ and use it to asymptotically learn the Whittle index.
	
	A common measure of performance of an RL algorithm is \emph{regret}, which measures the difference in performance of the learning algorithm that doesn't a priori know the system model with the performance of a baseline policy that knows the model. However, the regret is not characterized in the existing literature on RL for RBs~\cite{meshram2017restless,borkar2018learning,Fu2019Qlearning,avrachenkov2020whittle,robledo2022qwi}.
	
	There are some results on characterizing regret 
	for some specific instances of RBs: a model of multi-class queues
	arising in mobile edge computing~\cite{xiong2021learning} and a model for
	scheduling when to observe uncontrolled Markov chains arising in opportunistic
	spectrum access in cognitive radios~\cite{tekin2012online,ortner2012regret,liu2013learning,jung2019thompson,jung2019episodic,gafni2020learning}. However, the regret analyses in these papers exploit specific features of the model and are not applicable to general models. The main contribution of this paper is to characterize the regret of a general RL algorithm for general RBs.

	It is not possible to directly use existing RL algorithms that achieve near
	optimal regret in RBs. To explain why this is the case, we provide a short
	overview of characterizing regret in RL. Consider a general MDP (Markov
	decision process) with finite state space of size $\mathsf{S}$ and finite
	action space of size $\mathsf{A}$. It is shown in~\cite{jaksch2010near} that
	no learning algorithm can achieve a regret of less than
	$\tilde{\Omega}(\sqrt{\mathsf{S}\mathsf{A}D T})$, where $D$ is the diameter of
	the underlying MDP and $T$ is the time horizon for which the system runs.
	Several classes of algorithms have been proposed in the literature which
	achieve this lower bound up to a factor of $\sqrt{\mathsf{S}}$ and logarithmic
	terms. Broadly speaking, these regret optimal RL algorithms fall into two
	classes: optimism-under-uncertainty (OUU) and Thompson sampling (TS). Two
	types of regret bounds are provided: frequentist regret, which is a bound on
	the worst case regret with high probability and Bayesian regret, which is a
	bound on the average regret (with respect to a pre-specified prior). A summary
	of the regret bounds for various algorithms is shown in
	Table~\ref{tab:comparison}.
	
	\begin{table}[!t]
		\centering
		\begin{threeparttable}
			\caption{A comparison of the regret bounds of various algorithms}
			\label{tab:comparison}
			\begin{tabular}{@{}cccc@{}}
				\toprule
				Algorithm & Algorithm Type & Regret\tnote{a} & Regret Type \\
				\midrule
				UCRL2~\cite{jaksch2010near} & OUU & 
				$\tildeO(D\mathsf{S}\sqrt{\mathsf{A}T})$ & Frequentist \\
				REGAL~\cite{bartlett2009regal} & OUU & 
				$\tildeO(H\mathsf{S}\sqrt{\mathsf{A}T})$ & Frequentist \\
				SCAL~\cite{fruit2018efficient} & OUU &
				$\tildeO(H\sqrt{\Gamma \mathsf{S} \mathsf{A} T})$ & Frequentist \\
				\cite{agrawal2017posterior} & TS &
				$\tildeO(D\sqrt{\mathsf{S}\mathsf{A}T})$\tnote{b} & Frequentist \\
				\cite{zhang2019regret} & TS & 
				$\tildeO(\sqrt{H\mathsf{S}\mathsf{A}T})$ & Frequentist \\
				TSDE~\cite{ouyang2017learning} & TS &
				$\tildeO(H\mathsf{S}\sqrt{\mathsf{A}T})$ & Bayesian  \\
				\bottomrule
			\end{tabular}
			\begin{tablenotes}\footnotesize
				\item[a]
				In the column on regret bounds, $\Gamma$ is the maximum number of states
				that can be reached from any state, $D$ is the diameter of the MDP, $H$ is
				the span of the bias of the MDP. These are related as $\Gamma \le S$ and $H
				\le D$ (established in~\cite{bartlett2009regal}). 
				\item[b]
				It is pointed out in~\cite{zhang2019regret} that there is a mistake in
				the proof in \cite{agrawal2017posterior} and it is suggested that the
				bound of~\cite{agrawal2017posterior} may be loose by a factor of
				$\sqrt{\mathsf{S}}$. 
			\end{tablenotes}
		\end{threeparttable}
	\end{table}
	
	Each of these state-of-the-art algorithms has a regret that scales
	approximately as $\tildeO(\mathsf{S}\sqrt{\mathsf{A}T})$, which is prohibitively
	large when translated to the RB setting for reasons explained below. Consider
	a RB with $n$ arms where at most $m$ arms can be activated at a time. Let
	$\mathsf{S}_i$ denote the size of the state space of arm~$i \in \{1,\dots,
	n\}$. Such an RB can be modeled as an MDP where the size of the state space is
	$\prod_{i=1}^n \mathsf{S}_i$ and the size of the action space is
	$\binom{n}{m}$. Thus, the regret of using any of the algorithms described in
	Table~\ref{tab:comparison} on RBs will be
	\(
	\tildeO\Bigl( \prod_{i=1}^n \mathsf{S}_i \sqrt{ \binom{n}{m} T } \Bigr),
	\)
	which grows exponentially with the number $n$ of arms. In this paper, we provide a more nuanced characterization of the scaling of the regret with the number of alternatives.
	
	In particular, we propose a Thompson-sampling based learning algorithm for RB, which we call as \texttt{RB-TSDE}. This algorithm is inspired from the TSDE (Thompson sampling with dynamic episodes) algorithm~\cite{ouyang2017learning}. We show that for a RB with $n$ arms where $m$ of them can be chosen at a time, \texttt{RB-TSDE} has a Bayesian regret (with respect to the Whittle index policy with known dynamics) of $\tilde{\mathcal{O}}(n^2\sqrt{T})$ or $\tilde{\mathcal{O}}(nm\sqrt{T})$ depending on the assumptions on the per-step reward. Under an additional technical assumption, we obtain an alternative regret bound of $\tilde{\mathcal{O}}(n^{1.5}\sqrt{T})$ or $\tilde{O}(\max\{m\sqrt{n}, n\} \sqrt{T})$.
	
	The rest of the paper is organized as follows. In Sec.~\ref{sec:problem}, we formulate the learning problem for RB when the state transition probabilities of all arms are unknown and present the main results. In Sec.~\ref{sec:unknown}, we present the Thompson sampling with dynamic episodes for RB and provide an upper bound on the regret. In Sec.~\ref{sec:proofsketch}, we provide the proof outline and defer the details to the Appendix. In Sec.~\ref{sec:numerical}, we demonstrate a numerical example of the regret of our algorithm. In Sec.~\ref{sec:discussion}, we discuss relaxation and sufficient conditions of some of the assumptions, in addition to comparison with the optimal policy. Finally, we conclude in Sec.~\ref{sec:conclusion}.
	
	\paragraph*{Notation.}
	We use upper case variables $S$, $A$, etc., to denote random variables, the corresponding lower variables ($s$, $a$, etc.) to denote their realizations, and corresponding calligraphic letters ($\mathcal{S}$, $\mathcal{A}$, etc.) to denote set of realizations. Subscripts denote time and superscript denotes arm. Thus $S^i_t$ is the state of arm~$i$ at time~$t$. Bold letters denote collection of variables across all arms. Thus, $\boldsymbol{S}_t = (S^1_t, \dots, S^n_t)$ is the set of states of all arms at time~$t$. $S_{0:t}$ is a shorthand for $(S_0, \dots, S_t)$. We use $\EXP[\cdot]$ to denote expectation of a random variable, $\PR(\cdot)$ to denote probability of an event and $\IND(\cdot)$ to denote the indicator of an event. Let $[n] := \{1, \ldots, n\}$.
	
	For a given function $f \colon {\cal X} \to \reals$, the span-norm of $f$ is defined as $\SPAN(f) = \max_{x \in {\cal X}}f(x) - \min_{x \in {\cal X}}f(x)$. Given two metric spaces $(\mathcal{X}, d_X)$ and $(\mathcal{Y}, d_Y)$, the Lipschitz constant of function $f: \mathcal{X} \to \mathcal{Y}$ is defined by
	\[ L_f = \sup_{\substack{x_1, x_2 \in \mathcal{X} \\ x_1 \neq x_2}} \dfrac{d_Y(f(x_1), f(x_2))}{d_X(x_1, x_2)}. \]
	Let $\zeta_1$ and $\zeta_2$ denote probability measures on $(\mathcal{X}, d_X)$. Then, the Kantorovich distance between them is defined as
	\[ \mathcal{K}(\zeta_1, \zeta_2) = \sup_{f: L_f \leq 1} \bigg| \sum_{x \in \mathcal{X}} f(x) \zeta_1(x) - \sum_{x \in \mathcal{X}} f(x) \zeta_2(x) \bigg|. \]
	Given a metric space $(\mathcal{X}, d_X)$, $\DIAM(\mathcal{X}) = \sup\{d_X(x_1, x_2): x_1, x_2 \in \mathcal{X}\}$ denotes the diameter of the set~$\mathcal{X}$.
	
	\section{Model, Problem Formulation and Results} \label{sec:problem}
	\subsection{Restless Bandits}
	Restless bandits (RB) are a class of resource allocation problems where, at each time instant, a decision maker has to select $m$ out of $n$ available alternatives. Each alternative, which is also called an arm, is a controlled Markov process~$\langle\mathcal{S}^i, \mathcal{A}^i = \{0, 1\}, P^i, r^i\rangle$ where $\mathcal{S}^i$ is the state space, $\mathcal{A}^i = \{0, 1\}$ is the action space, $P^i: \mathcal{S}^i \times \mathcal{A}^i \to \Delta(\mathcal{S}^i)$ is the controlled transition matrix, and $r^i: \mathcal{S}^i \times \mathcal{A}^i \to [0, R_{\max}]$ is the per-step reward. The action $A^i_t = 1$ means the decision maker selects arm~$i$ at time~$t$. The arms for which $A^i_t = 1$ are called \textit{active} arms and the arms for which $A^i_t = 0$ are called \textit{passive} arms.
	
	Let ${\bs{\mathcal{S}}} = {\bs{\mathcal{S}}^1} \times \dots \times {\bs{\mathcal{S}}}^n$
	denote the joint state space and $ {\bs{\mathcal{A}}}(m) = \bigl\{ {\bs a}
	\in \{0, 1\}^{n} : \lVert \bs{a} \rVert_1 = m \bigr\} $ denote the feasible action space, where $\lVert \bs{a} \rVert_1 \coloneqq \sum_{i \in [n]} \lvert a^i \rvert = \sum_{i \in [n]} a^i$. We assume that the initial state $\bs S_0 = (S^1_0, \ldots, S^n_0)$ is a random variable which is independent across arms and has a known initial distribution. Moreover, the arms evolve independently, i.e., for any ${\bs s}_{0:t} = (s^1_{0:t}, \ldots, s^n_{0:t})$ and ${\bs a}_{0:t} = (a^1_{0:t}, \ldots, a^n_{0:t})$, we have
	\begin{multline*} \label{eqn:dynamics}
		\PR \left( {\bs S}_{t+1} = {\bs s}_{t+1} | {\bs S}_{0:t} = {\bs s}_{0:t}, {\bs A}_{0:t} = {\bs a}_{0:t} \right) \\
		= \prod_{i = 1}^{n} P^i \left( s^{i}_{t+1} | s^i_{t}, a^i_t \right) := \bs{P}({\bs s}_{t+1} | {\bs s}_{t}, {\bs a}_{t}).
	\end{multline*}
	
	We consider two reward models:
	\begin{itemize}
		\item \textbf{Model A:} All arms, whether active or not, yield rewards, i.e., the aggregated per-step reward is given by 
		\(
		\bs{r}({\bs s}_t, {\bs a}_t) = \sum_{i \in [n]} r^i(s^i_t, a^i_t).
		\)
		\goodbreak	
		\item \textbf{Model B:} Only the activated arms yield rewards. The state of the passive arms evolves, but the arms do not yield reward. Thus, the aggregated per-step reward is given by 
		\(
		\bs{r}({\bs s}_t, {\bs a}_t) = \sum_{i \in [n]} r^i(s^i_t, a^i_t) \IND(\{ a^i_t = 1\}).
		\)
	\end{itemize}
	Note that Model B is same as Model A under the assumption that $r^i(\,\cdot\,, 0) = 0$ for all arms~$i \in [n]$. For that reason, for most of the paper, we will take 
	\(
	\bs{r}({\bs s}_t, {\bs a}_t) = \sum_{i \in [n]} r^i(s^i_t, a^i_t)
	\)
	and assume $r^i(\,\cdot\,,0) = 0$ when specializing for results of Model~B.
	
	\begin{remark}
		Both Models~A and~B arise in different applications. Examples of Model~A include queuing networks~\cite{borkar2017opportunistic}, where all queues incur holding cost; and machine maintenance~\cite{glazebrook2006some}, where all machines incur a cost when run in a faulty state. Examples of Model~B include cognitive radios~\cite{liu2010indexability}, where the reward depends only on the state of the selected channels.
	\end{remark}
	
	Let $\bs{\Pi}$ denote the family of all possible (potentially history dependent and randomized) policies for the decision maker (who observes the state of all arms). The performance of any policy $\bs{\pi} \in \bs{\Pi}$ is given by
	\begin{equation}
		J({\bs \pi}) \coloneqq \liminf\limits_{T \to \infty} \dfrac{1}{T} \EXP\biggl[ \sum_{t = 1}^{T} {\bs r}({\bs S}_t, {\bs A}_t) \biggr], \label{eqn:obj_func}
	\end{equation}
	where the expectation is taken with respect to the initial state distribution and the joint distribution induced on all system variables.
	
	The objective of the decision maker is to choose a policy $\bs{\pi} \in \bs{\Pi}$ to maximize the total expected reward $J(\bs{\pi})$. This objective is a MDP but computing an optimal policy using dynamic program suffers from the curse of dimensionality. For example, if $|\mathcal{S}^i| = \mathsf{S}$ for each $i \in [n]$, then $|{\bs{\mathcal{S}}}| = \mathsf{S}^n$ and $|{\bs{\mathcal{A}}}(m)| = {n \choose m}$. Then, the computational complexity of each step of value iteration is $|{\bs{\mathcal{A}}}(m)| |{\bs{\mathcal{S}}}|^2 = {n \choose m} \mathsf{S}^{2n}$, which is prohibitively large for even moderate values of $\mathsf{S}$ and $n$. For this reason, most of the RB literature focuses on a computationally tractable but sub-optimal approach known as the Whittle index policy. 
	%which we explain in Sec.~\ref{sec:background}.
	
	\subsection{Whittle index policy}
	The Whittle index policy is motivated by the solution of a relaxation of the original optimization problem. Instead of the hard constraint of activating exactly $m$ arms at a time, consider a relaxation where $m$ arms have to be activated on average, i.e.,
	\begin{align} 
		& \max_{{\bs \pi} \in {\bs \Pi}} \liminf\limits_{T \to \infty} \dfrac{1}{T} \EXP\biggl[ \sum_{t = 1}^{T} \bs{r}({\bs S}_t, {\bs A}_t) \biggr], \notag \\
		& \text{s.t. } \limsup\limits_{T \to \infty} \dfrac{1}{T} \EXP\biggl[ \sum_{t = 1}^{T} \lVert {\bs A}_t \rVert_1 \biggr] = m. \label{eqn:opt-prob}
	\end{align}
	Note that this relaxation is simply used to obtain a
	decomposition to define Whittle indices. The Whittle index policy, which is
	stated at the end of this section, picks exactly $m$ arms at each time
	step.
	
	Following~\cite{beutler1985optimal}, the Lagrangian relaxation of~\eqref{eqn:opt-prob} is given by
	\begin{equation}
		\mathcal {L}(\lambda,\bs{\pi}) = \liminf\limits_{T \to \infty} \dfrac{1}{T} \EXP\biggl[ \sum_{t = 1}^{T} \Bigl[
		{\bs r}({\bs S}_t, {\bs A}_t) - \lambda \lVert {\bs A}_t \rVert_{1} \Bigr] \biggr].
		\label{eqn:Lagrangian}
	\end{equation}
	Moreover, a policy $\bs{\pi}^* \in \bs{\Pi}$ is optimal if there exists a Lagrange multiplier $\lambda \ge 0$ such that $\bs{\pi}^* \in \arg\max_{\bs{\pi} \in \bs{\Pi}} \mathcal{L}(\lambda, \bs{\pi})$ and the expected number of activations under $\bs{\pi}^*$ is exactly~$m$. See~\cite[Theorem 4.3]{beutler1985optimal}.
	
	Note that the optimization problem $\max_{\bs{\pi} \in \bs{\Pi}} \mathcal{L}(\lambda, \bs{\pi})$ is decoupled across arms because the per-step reward is decoupled:
	\[
	\bs{r}({\bs S}_t, {\bs A}_t) - \lambda \lVert {\bs A}_t \rVert_1 = \sum_{i \in [n]} \Bigl[ r^i(S^i_t, A^i_t) - \lambda A^i_t \Bigr].
	\]
	Therefore, for a given $\lambda$, $\max_{\bs \pi \in \bs \Pi} \mathcal{L}(\lambda, \bs \pi)$ is equivalent to the following $n$ decoupled optimization problems: for all $i \in [n]$, 
	\begin{equation} \label{eqn:opt-lag-prob}
		\max_{\pi^i: \mathcal{S}^i \to \{0, 1\}} \liminf\limits_{T \to \infty} \dfrac{1}{T} \EXP\biggl[ \sum_{t = 1}^{T} \Bigl[r^i(S^i_t, A^i_t) - \lambda A^i_t\Bigr] \biggr].
	\end{equation}
	
	Let $\pi^i_\lambda$ denote the optimal policy for Problem~\eqref{eqn:opt-lag-prob}.
	Define the \emph{passive set} $\mathcal{W}^i_\lambda$ as the set of states for which the optimal policy~$\pi^i_\lambda$ prescribes passive action, i.e., $\mathcal{W}^i_\lambda := \bigl\{ s \in \mathcal{S}^i: \pi^i_\lambda(s) = 0 \bigr\}$.
	
	\begin{definition}[Indexability and Whittle index] \label{def:idxbl}
		A RB is said to be \textbf{indexable} if $\mathcal{W}^i_\lambda$ is non-decreasing in $\lambda$, i.e., for any $\lambda_1, \lambda_2 \in \mathbb{R}$ such that $\lambda_1 \leq \lambda_2$, we have $\mathcal{W}^i_{\lambda_1} \subseteq \mathcal{W}^i_{\lambda_2}$. For an indexable RB, \textbf{the Whittle index}~$w^i(s)$ of state~$s \in \mathcal{S}^i$ is the smallest value of $\lambda$ for which state~$s$ is part of the passive set $\mathcal{W}^i_\lambda$, i.e.,
		\begin{equation*}
			w^i(s) = \inf \left\{ \lambda \in \mathbb{R}: s \in \mathcal{W}^i_\lambda \right\}.
		\end{equation*}
	\end{definition}
	
	Note that if the penalty $\lambda = w^i(s)$, then the policy $\pi^i_\lambda$ is indifferent between taking passive or active actions at state~$s$.
	
	The Whittle index policy is a feasible policy for the original optimization problem and is given as follows: \emph{At each time, activate the arms with the $m$ largest values of the Whittle index at their current state.} 
	
	As argued in \cite{whittle1988restless}, the Whittle index policy is meaningful only when all arms are indexable. Various sufficient conditions for indexability are available in the literature~\cite{glazebrook2006some,ninomora2007,akbarzadeh2022conditions}. In some settings, the Whittle index policy is optimal~\cite{gittins1979bandit,weber1990index,lott2000optimality}. For general models, there is also strong evidence to suggest that 
	the Whittle index policy performs close to optimal~\cite{glazebrook2006some,ninomora2007,akbarzadeh2022conditions,avrachenkov2013congestion,wang2019whittle}. Algorithms to efficiently compute Whittle indices are presented in \cite{nino2007dynamic,akbarzadeh2022conditions}.
	
	\subsection{The learning problem}
	Let $\bs{\mu}$ denote the Whittle index policy and ${\bs J}(\bs{\mu})$ denote its performance. We are interested in a setting where the transition matrices $\{P^i\}_{i \in [n]}$ of the arms are unknown but the decision maker has a prior on them. In this setting, the performance of a policy ${\bs \pi}$ operating for horizon~$T$ is characterized by the Bayesian regret given by
	\begin{equation}\label{eqn:regret}
		\REGRET(T; {\bs \pi}) 
		= \EXP^{\bs \pi}\biggl[ 
		T {\bs J}(\bs{\mu}) 
		- 
		\sum_{t = 1}^T \bs{r}({\bs S}_t, {\bs A}_t) 
		\biggr], 
	\end{equation}
	where the expectation is taken with respect to the prior distribution on $\{P^i\}_{i \in [n]}$, the initial condition, and the potential randomization done by the policy~${\bs \pi}$. Bayesian regret is a well-known metric used in various setting \cite{rusmevichientong2010linearly,agrawal2013thompson,russo2014learning,ouyang2017learning}. An alternative method to quantify the performance is the frequentist regret but we focus on Bayesian regret for a comparison of the two notions of the regret, we refer the reader to \cite{kaufmann2011efficiency,kaufmann2014analysis}.
	
	\begin{remark}
		We measure the regret with respect to the Whittle index policy. In contrast, in most of the existing research, regret is defined with reference to the optimal policy. In principle, the results presented in this paper are also applicable to regret defined with respect to the optimal policy provided it is possible to compute the optimal policy for a given model. See Sec.~\ref{subsec:optimal} for details.
	\end{remark}
	
	\begin{remark} \label{rem:2}
		The rested multi-armed bandit problem is a special case of Model~B where passive arms are frozen, i.e., $P^i(s_{+}|s,0) = \IND(\{s_+ = s\})$ for all arms~$i \in [n]$. For this model, the Whittle index policy reduces to what is called the \emph{Gittins index policy} and is optimal~\cite{gittins1979bandit}. Thus, the results obtained in this paper are also applicable to the rested multi-armed bandits.
	\end{remark}
	
	\subsection{The main results}
	
	Our main result is to propose a Thompson-sampling based algorithm, which we call \texttt{RB-TSDE}, and characterize its regret. In particular, let $\mathsf{S}^i = |\mathcal{S}^i|$ denote the size of the state space of arm~$i$ and $\bar {\mathsf{S}}_n = \sum_{i \in [n]} \mathsf{S}^i$ denote the sum of the sizes of the state space of all arms. Then, we show the following.
	
	\begin{result}
		The regret of {\tt RB-TSDE} is bounded by
		\[
		\REGRET(T; {\tt RB-TSDE}) \le 
		\OO\bigl(\alpha \bar{\mathsf{S}}_n \Sqrt{T \log T }\bigr),
		\]
		where $\alpha = n$ for Model~A and $\alpha = m$ for Model~B. Under
		additional assumptions, the bound for both models can be tightened to
		\[
		\REGRET(T; {\tt RB-TSDE}) \le 
		\OO\bigl(\max\{\alpha\bar{\mathsf{S}}_n, \bar{\mathsf{S}}_n\} \Sqrt{ T \log T }\bigr).
		\]
	\end{result}
	The detailed characterization of the constants in the $\OO(\cdot)$ terms is given in Theorem~\ref{thm:main} and Theorem~\ref{thm:lipschitz} later.

	%%%%%%%%%%%%%%%%%%%%%%%%%%%%%%%%%%%%%%%%%%%%%%%%%%%%%%%%%%%%%%%%%%%%%%%%%%%%%%%%%%%%%%%%%%%%%%%%%%%%%%%%%%%%%%%%%%%%%%%%%%%%%%%%%%%%%%%%%%%%%%%%%%%%%%%%%%%%%%%%%%%%%%%%%%%%%%%%%%%%%%%%%%%%%%%%%%%%%%%%%%%%%%%%%%%%%%%%%%%%%%%%%%%%%%%%%%%%%%%%%%%%%%%%%%%%%%%%%%%%%%%%%%%%%%%%%%%%%%%%%%%%%%%%%%%%%%%%%%%%%%%%%%%%%%%%%%%%%%%%%%%%%%%%%%%%%%%%%%%%%%%%%%%%%%%%%%%%%%%%%%%%%%%%%%%%%%%%%%%%%%%%%%%%%%%%%%%%%%%%%%%%%
	
	\section{Learning Algorithm for RB} \label{sec:unknown}
	
	\subsection{Assumptions on the unknown parameters.}
	Let $\theta^i_\star$ denote the unknown parameters of the transition matrices
	$[P^i(\cdot | \cdot, 0) ~ P^i(\cdot | \cdot, 1)]$, $i \in [n]$. We assume that $\theta^i_\star$ belongs to a
	compact set~$\Theta^i$. We impose the following assumptions on the model.
	
	\begin{assumption}\label{assump:indexable}
		For any~$i \in [n]$ and $\theta^i \in \Theta^i$, the RB $\langle \mathcal{S}^i, \mathcal{A}^i = \{0, 1\}, P^i(\theta^i), r^i \rangle$ is indexable.
	\end{assumption}
	
	\begin{assumption} \label{assump:acoe}
		Let $\bs{\mu}_{\theta}$ denote the Whittle index policy corresponding to model~$\theta$. Let $\bs{P}_{\theta}$ denote the controlled transition matrix under policy $\bs{\mu}_{\theta}$ and $J_\theta$ denote the average reward of policy $\bs{\mu}_{\theta}$. We assume that for every $\theta \in \Theta$, $J_\theta$ does not depend on the initial state and also assume that there exists a bounded differential value function $\bs{V}_\theta$ such that $(J_\theta, \bs{V}_{\theta})$ satisfy the average reward Bellman equation:
		\begin{equation}\label{eqn:Bellman}
			J_\theta + \bs{V}_\theta(\bs{s}) = 
			\bs{r}(\bs{s}, \bs{\mu}_{\theta}(\bs{s})) + \bigl[ \bs{P}_\theta \bs{V}_\theta \bigr](\bs{s}),
			\quad \forall \bs{s} \in \bs{\mathcal{S}}.
		\end{equation}
	\end{assumption}
	
	Assumption~\ref{assump:indexable} is necessary for the Whittle index heuristic to be meaningful. Assumption~\ref{assump:acoe} ensures that the average reward of the Whittle index policy is well defined for all models. There are various sufficient conditions for Assumption~\ref{assump:acoe} in the literature. See~\cite{arapostathis1993} for an overview. 
	
	The average reward Bellman equation~\eqref{eqn:Bellman} has an infinite number of solutions. In particular, if $(J_\theta, \bs{V}_\theta)$ satisfies~\eqref{eqn:Bellman}, then so does $(J_\theta, \bs{V}_\theta + \text{constant})$. Assumption~\ref{assump:acoe} implies that $\SPAN(\bs{V}_\theta)$ is bounded. A bound on $\SPAN(\bs{V}_\theta)$ under a different set of assumptions is presented in \cite{bartlett2009regal} but this bound does not suffice for our analysis. See Remark~\ref{rem:6} for details. As we want to capture the scaling of $\SPAN(\bs{V}_\theta)$ with $n$ and $m$, so we impose an additional assumption on the model.
	
	The ergodicity coefficient of $\bs{P}_\theta$ is defined as
	\[
	\lambda_{\bs{P}_\theta} = 1 - \min_{\bs{s},\bs{s'}\in \bs{\mathcal{S}}}
	\sum_{\bs{z} \in \bs{\mathcal{S}}} \min\{ \bs{P}_\theta(\bs{z} | \bs{s}), 
	\bs{P}_\theta(\bs{z} | \bs{s}') \}.
	\]
	We impose
	the following assumption on the model.
	\begin{assumption} \label{assmp:3}
		We assume there exists $\lambda^* < 1$ such that $\sup_{\theta \in \Theta} \lambda_{\bs{P}_\theta} \le \lambda^*$. 
	\end{assumption}
	See~\cite[Sec.~5]{arapostathis1993} for various equivalent characterizations
	of $\lambda_{\bs{P}_\theta} < 1$. Assumption~\ref{assmp:3} is used while
	analyzing the rate of convergence of relative value
	iteration~\cite{morton1977discounting} to bound the span of the value
	function. A relaxation of Assumption~\ref{assmp:3} is presented in Sec.~\ref{sec:relax}.
	
	% a standard
	% assumption in the analysis of average cost MDPs with countable or unbounded
	% state space. We impose this assumption because it enables us to bound
	% $\SPAN(\bs{V}_\theta)$ in terms of the number of arms.
	%We impose Assumption~\ref{assmp:3} because it implies $\SPAN(\bs{P}_\theta \bs{V}_\theta) \leq \lambda_{\bs{P}_\theta} \SPAN(\bs{V}_\theta)$ \cite[Proposition 6.6.1]{puterman2014markov}. We use this relationship to obtain a bound on $\SPAN(\bs{V}_\theta)$ in Lemma~\ref{lem:bound} below.
	
	The ergodicity coefficient of $\bs{P}_\theta$ depends on the Whittle index
	policy $\bs{\mu}_\theta$. A policy independent upper bound of the ergodicity
	coefficient is given by the contraction factor, which is defined as:
	\[
	\lambda' = 1 - \min_{\substack{\bs{s},\bs{s'}\in \bs{\mathcal{S}}, \\ \bs{a} \in
			\bs{\mathcal{A}}(m), \bs{a}' \in \bs{\mathcal{A}}(m)}}
	\sum_{\bs{z} \in \bs{\mathcal{S}}} \min\{ \bs{P}(\bs{z} | \bs{s}, \bs{a}), 
	\bs{P}(\bs{z} | \bs{s}', \bs{a}') \}
	\]
	Since the dynamics of the arms are independent, the definition of contraction
	factor implies that a sufficient condition for Assumption~\ref{assmp:3} is that for every arm, and every pair of state-action pairs, there exists a next state that can be reached from both state-action pairs with positive probability in one step. Moreover, if for every arm there is a distinguished state which can be reached from all state-action pairs with probability at least $\varepsilon$, then the ergodicity coefficient is less than $1-\varepsilon$. 
	
	% As we assume the dynamics of each arm is independent, we get
	% \begin{align*}
		% 	\lambda_{\bs{P}_\theta} & = 1 - \sum_{\bs{z} \in \mathcal{S}} \min_{\bs{s} \in \bs{\mathcal{S}}, \bs{a} \in \bs{\mathcal{A}}(m)} \prod_{i \in [n]} P^i_\theta(z^i | s^i, a^i) = 1 - \sum_{\bs{z} \in \mathcal{S}} \prod_{i \in [n]} \min_{s^i \in \mathcal{S}^i, a^i \in \{0, 1\}} P^i_\theta(z^i | s^i, a^i) \notag \\
		% 	& = 1 - \prod_{i \in [n]} \sum_{z^i \in \mathcal{S}^i} \min_{s^i \in \mathcal{S}^i, a^i \in \{0, 1\}} P^i_\theta(z^i | s^i, a^i)
		% \end{align*}
	% Then, we impose the following assumption on the problem.
	% A sufficient condition is for any $\theta \in \Theta$, $s^i \in \mathcal{S}^i$, $a^i \in \{0, 1\}$ and $i \in [n]$, there exists $z^i \in \mathcal{S}^i$ such that $P^i_\theta(z^i | s^i, a^i) > 0$.
	
	\subsection{Priors and posterior updates.}
	We assume that $\{\theta^i_\star\}_{i \in [n]}$ are independent random variables and use $\phi^i_1$ to denote the prior on $\theta^i$ for each arm~$i \in [n]$. At time~$t$, let $h^i_t = (s^i_1, a^i_1, \dots,\allowbreak s^i_{t-1}, a^i_{t-1}, s^i_t)$ denote the history of states and actions at arm~$i$. Let $\phi^i_t$ denote the posterior distribution on $\theta^i_\star$ given $h^i_t$. Then, upon applying action $a^i_t$ and observing the next state~$s^i_{t+1}$, the posterior distribution $\phi^i_{t+1}$ can be computed using Bayes rule as 
	\begin{equation} \label{eqn:phi-update}
		\phi^i_{t+1}(d \theta) = \dfrac{P^i(s^i_{t+1} | s^i_t, a^i_t; \theta) \phi^i_{t}(d \theta)}{\int P^i(s^i_{t+1} | s^i_t, a^i_t; \tilde \theta) \phi^i_{t}(d \tilde \theta)}.
	\end{equation}
	If the prior is a conjugate distribution on $\Theta^i$, then the posterior can be updated in closed form. Note that the exact form of the prior and the posterior update rule are not important for the description of the algorithm or its regret analysis.
	
	\subsection{\texttt{RB-TSDE} Algorithm (Distributed implementation)}
	We propose a Thompson-sampling based algorithm, which we call \texttt{RB-TSDE}. Our algorithm is inspired by the \texttt{TSDE} (Thompson Sampling with Dynamic Episodes) algorithm of~\cite{ouyang2017learning}.
	
	The \texttt{RB-TSDE} algorithm consists of a coordinator and $n$ actors, one
	for each arm. The coordinators and the actors require synchronized
	communication as described below. The
	whole algorithm is described in Alg.~\ref{Alg:RB-TSDE}. 
	
	\begin{algorithm}[t!]
		\caption{\textsc{RB-TSDE}}
		\begin{algorithmic}[1]
			\STATE \textbf{Input: } $\{(\Theta^i, \phi^i_1)\}_{i \in [n]}$.
			\STATE Initialize $t \leftarrow 1$, $t_1 \leftarrow 0$, $T_0 \leftarrow 0$, $k \leftarrow 0$, $N^i(s^i, a^i) = 0, \forall a^i \in \{0, 1\}, s^i \in \mathsf{S}^i, \forall i \in [n]$, $\theta_0$, $\mu_{\theta_0}$.
			\FOR{$t = 1, 2, \ldots$}
			\IF{$t_k + T_{k-1} < t$ or $2 N^i_{t_k}(s, a) < N^i_t(s, a)$ for any $s \in {\mathcal{S}}^i$, $a \in \{0, 1\}$, $i \in [n]$}
			\STATE Set $k \leftarrow k + 1$, $T_{k-1} \leftarrow t - t_{k-1}$, $t_k \leftarrow t$.
			\STATE Actor $i \in [n]$ samples $\theta^i_{k} \sim \phi^i_{t_k}$ and
			compute $w^i_{t_k}$.
			%\STATE Update $\mu_{\theta_k}$ by \cite[Algorithm~$2$]{akbarzadeh2022learning}.
			\ENDIF
			\STATE Actor $i \in [n]$ sends the Whittle index $w^i_{t_k}(s^i_t)$
			to the coordinator.
			\STATE The coordinator sends $a^i_t = 1$ to the arms
			with the $m$-highest Whittle index and sends $a^i_t = 0$ to others.
			\STATE Actor $i \in [n]$ updates $\phi^i_{t+1}$ according to
			\eqref{eqn:phi-update}.
			\ENDFOR
		\end{algorithmic} 
		\label{Alg:RB-TSDE}
	\end{algorithm}

	As the name suggests, \texttt{RB-TSDE} operates in episodes of dynamic length.
	The episodes are synchronized for all actors and the coordinator signals the
	start of episodes to all actors. The actor at arm~$i$ maintains a posterior
	$\phi^i_t$ distribution on the dynamics of arm~$i$ according
	to~\eqref{eqn:phi-update} and keep track of \( N^{i}_t(s^i, a^i) = \sum_{\tau
		= 1}^{t-1} \IND(\{(S^i_{\tau}, A^i_{\tau}) = (s^i, a^i)\}). \) 
	
	Let $t_k$ and $T_k$ denote the start time and length of episode~$k$,
	respectively. The end of the episode can either be triggered by the
	coordinator or any of the actors. The coordinator triggers the end of the
	episode if the length of the episode is one more than the length of the
	previous episode. The actor for arm~$i$ triggers the end of the episode if the
	number of state-action visits $N^i_t(s^i_t, a^i_t)$ of the current state-action
	pair are more than double of their value at the beginning of the episode.
	Thus, 
	\begin{align*}
		t_{k+1} = \min \bigl\{ & t > t_k : t - t_k > T_{k-1} \text{ or } \\
		& N^i_t(s^i, a^i) > 2 N_{t_k}(s^i, a^i) \text{ for some } (i, s^i, a^i)\bigr\}.
	\end{align*}

	At the beginning of episode~$k$, the actor for arm $i \in [n]$
	samples a parameter $\theta^i_k$ from the posterior~$\phi^i_{t_k}$ and
	computes the Whittle index $w^i_{t_k}$ for all states. During
	episode~$k$, at each time~$t$, the actor at arm~$i$ sends the value of
	$w^i_{t_k}(s^i_t)$ to the coordinator. The coordinator receives
	$w^i_{t_k}(s^i_t)$ from all arms, sends the active action $a^i_t = 1$ to the
	arms with the $m$-highest values of the Whittle index, and sends the passive
	action $a^i_t = 0$ to the remaining arms. This process continues until a
	condition for ending the episode is triggered by the coordinator or one of the
	actors.
	
	\subsection{Regret bound}
	\begin{theorem}\label{thm:main}
		Under Assumptions~\ref{assump:indexable}--\ref{assmp:3}, the regret of {\tt RB-TSDE} is upper bounded as follows:
		\[
		\REGRET(T; {\tt RB-TSDE}) 
		< 40 \alpha \frac{R_{\max}}{1-\lambda^*} \bar {\mathsf{S}}_n \sqrt{T \log T},
		\]
		where $\alpha = n$ for Model~A and $\alpha = m$ for Model~B. 
	\end{theorem}
	%\begin{corollary}
	%	If the state space size of each arm is $\mathsf{S}$, the regret of {\tt RB-TSDE} is upper bounded as follows: for Model~A,
	%	\[
	%	\REGRET(T; {\tt RB-TSDE}) 
	%	\le 20 n^2 R_{\max} \mathsf{S}
	%	\Sqrt{n \mathsf{S} T \log (T)}.
	%	\]
	%	and for model~B
	%	\[
	%	\REGRET(T; {\tt RB-TSDE}) 
	%	\le 20 nm R_{\max} \mathsf{S}
	%	\Sqrt{n \mathsf{S}T \log (T)}.
	%	\]
	%\end{corollary}
	See Sec.~\ref{subsec:final-bound} for proof.
	
	%	\begin{algorithm}[t!]
		%		\caption{\textsc{RB-TSDE}}
		%		\begin{algorithmic}[1]
			%			\STATE \textbf{Input: } $\{(\Theta^i, \xi^i)\}_{i \in [n]}$.
			%			\STATE Initialize $t = 0$, $k = 1$, $t_0 = 0$, $t_1 = 1$.
			%			\STATE Initialize $(\theta^1_0, \dots, \theta^n_0)$ randomly.
			%			\WHILE{$k \geq 1$}
			%			\STATE Set $t_k = t$ and $T_{k-1} = t_k - t_{k-1}$.
			%			\STATE For all arm~$i \in [n]$, draw samples $\hat{P}^{i}_{t_k}(s, a, \dot)$ from posterior distribution with parameters $\theta^{i}_{t-1}$ for all arm-state-action tuples.
			%			\STATE Compute $\hat{W}^i(s)$ for all $s \in \mathcal{S}$ by Algorithm~\ref{Alg:Whittle index policy} given in Appendix~\ref{app:whittle}.
			%			\WHILE{$t \leq t_k + T_{k-1}$ or $N^i_t(s, a) \leq 2 N^i_{t_k}(s, a)$ for all state action pairs~$(s, a)$ of all arms~$i \in [n]$}
			%			\STATE Take actions prescribed by policy $\hat{\bs{\pi}}_W({\bs S}_t)$, i.e., $A^i_t = 1$ for arms which are among the $m$ largest $\hat{W}^i(S^i_t)$ and $A^i_t = 0$ for the rest of the arms, at each time~$t$.
			%			\STATE Observe $S^i_{t+1}$ for all arms~$i \in [n]$.
			%			\STATE Update $N^{i}_{t+1}(s, a, s')$ according to \eqref{eqn:counter_update} for the realized state-action-state pairs for all arms.
			%			\STATE Set $t = t+1$.
			%			\ENDWHILE
			%			\STATE Set $k = k+1$.
			%			\ENDWHILE
			%		\end{algorithmic} 
		%		\label{Alg:Whittle index policy-TSDE}
		%	\end{algorithm}
	
	We derive a tighter regret bound under a stronger assumption. We first assume that the state space of each arm~$\mathcal{S}^i$, $i \in [n]$, is a finite subset of $\mathbb{R}$ and use $d^i$ to denote the Euclidean metric on $\mathbb{R}$, i.e., $d^i(s, s') = |s - s'|$. Furthermore, let $\bs{d}_\PP(\bs{s}, \bs{s}') = \bigl( \sum_{i \in [n]} d^i(s^i,s^{\prime,i})^\PP \bigr)^{1/\PP}$ for any $\bs{s}, \bs{s}' \in \bs{\mathcal{S}}$. We then impose the following assumption.
	\begin{assumption}\label{assump:lipschitz}
		For each $\theta \in \Theta$, the value function $\bs{V}_\theta$ is Lipschitz with a Lipschitz constant upper bounded by $\bs{L}_v$.
	\end{assumption}
	
	In general, Assumption~\ref{assump:lipschitz} depends on the specific model being considered. We present one instance where Assumption~\ref{assump:lipschitz} is satisfied in Sec.~\ref{sec:sufficient}.
	
	\begin{theorem}\label{thm:lipschitz}
		Under Assumptions~\ref{assump:indexable}--\ref{assump:lipschitz}, the regret of {\tt RB-TSDE} for both Model~A and Model~B is upper bounded as follows:
		\begin{multline*}
			\REGRET(T; {\tt RB-TSDE}) \\
			< 12 \max\{\alpha \sqrt{\bar {\mathsf{S}}_n}, \bar {\mathsf{S}}_n \} \max\Bigl\{ \frac{ R_{\max}}{1-\lambda^*}, \mathsf{D}_{\max} \bs{L}_v \Bigr\} \\
			\sqrt{KT \log (T\max\{1, K'\})}
		\end{multline*}
		where $\alpha = n$ for Model~A and $\alpha = m$ for Model~B, $K$ and $K'$ are positive constants independent of $n$ and $T$, and $\mathsf{D}_{\max} = \max_{i \in [n]} \DIAM(\mathcal{S}^i)$.
	\end{theorem}
	See Sec.~\ref{subsec:final-bound} for proof.
	
	\begin{remark}
		If we directly use an existing RL algorithm for RBs, the regret will scale as $\tilde{\OO}(2^n \sqrt{T})$ or larger. The results of Theorems~\ref{thm:main} and~\ref{thm:lipschitz} show that the regret scales as either $\tilde{\OO}(n^2 \sqrt{T})$, $\tilde\OO(n^{1.5} \sqrt T)$, or $\tilde\OO(n \sqrt T)$ depending on the modeling assumptions. Thus, using a learning algorithm which is adapted to the structure of the models gives a significantly better scaling with the number of arms. 
	\end{remark}
	
	\begin{remark}\label{rem:scaling}
		The exact scaling with the number of arms depends on how $m$ scales with $n$. For example, if $m$ remains constant, then under Assumptions~\ref{assump:indexable}--\ref{assmp:3}, the regret for Model~A scales as $\tilde{\OO}(n^2 \sqrt T)$, while the regret for Model~B scales as $\tilde{\OO}(n \sqrt T)$. Under Assumption~\ref{assump:lipschitz}, the regret for Model~A scales as $\tilde{\OO}(n^{1.5} \sqrt T)$, while the regret for Model~B scales as $\tilde{\OO}(n\sqrt{T})$. On the other hand, if $m$ scales as $\beta n$, where $\beta < 1$ is a constant\footnote{For this setting, it was shown in~\cite{weber1990index} that the Whittle index policy is optimal as $n \to \infty$.}, then under Assumptions~\ref{assump:indexable}--\ref{assmp:3}, the regret for both models scales as $\tilde{\OO}(n^2 \sqrt{T})$. Under Assumption~\ref{assump:lipschitz}, the regret for both models scales as $\tilde{\OO}(n^{1.5} \sqrt T)$. Thus, for Model~A, the regret bound of Theorem~\ref{thm:lipschitz} is tighter than that of Theorem~\ref{thm:main}, but for Model~B, it depends on the scaling assumptions on~$m$.
	\end{remark}
	
	\section{Regret Analysis} \label{sec:proofsketch}
	
	The high level idea of the proof is similar to the analysis in \cite{ouyang2017learning} but we exploit the properties of the RB model while simplifying individual terms. We first start with bounds on the average reward and the differential value function.
	
	\subsection{Bounds on average reward and differential value function.}
	As mentioned earlier, $\bs{V}_\theta$ is unique only up to an additive constant.
	We assume that $\bs{V}_\theta$ is chosen such that $\bs{\xi}_\theta \bs{V}_\theta = 0$, where $\bs{\xi}_\theta$ is the stationary distribution of $\bs{P}_\theta$. This ensures that $\bs{V}_\theta$ is equal to the \emph{asymptotic bias} of policy $\bs{\mu}_{\theta}$ and is given by
	\begin{equation}\label{eq:bias}
		\bs{V}_\theta = \sum_{t=0}^\infty \bs{P}_\theta^t \bigl[
		\bs{r} - J_\theta \bigr].
	\end{equation}
	See for example~\cite{puterman2014markov}.
	
	Then we have the following bounds.
	\begin{lemma}\label{lem:bound}
		Under Assumption~\ref{assmp:3}, for any $\theta \in \Theta$,
		\[
		0 \le J_\theta \le \alpha R_{\max}
		\text{ and }
		\SPAN(\bs{V}_\theta) \le 2 \alpha R_{\max}/(1 - \lambda^*),
		\]
		where $\alpha = n$ for Model~A and $\alpha = m$ for Model~B. 
	\end{lemma}
	\begin{proof}
		Note that for Model~A, $\bs{r}(\bs{s},\bs{a}) \in [0, n R_{\max}]$ while for Model~B, $\bs{r}(\bs{s},\bs{a}) \in [0, m R_{\max}]$.
		Then, the bounds for $J_\theta$ follow immediately from definition. The bounds on $\SPAN(\bs{V}_\theta)$ follow from~\eqref{eq:bias}, $\SPAN(\bs{s} + \bs{y}) \le \SPAN(\bs{s}) + \SPAN(\bs{y})$, Assumption~\ref{assmp:3}, and the fact that for any vector~$\bs{v}$, $\SPAN(\bs{P}_\theta \bs{v}) \leq \lambda_{\bs{P}_\theta} \SPAN(\bs{v})$.
	\end{proof}
	
	\begin{remark}
		Lemma~\ref{lem:bound} shows the key difference between Models~A and~B. When all arms yield rewards, the maximum value of $\bs{r}(\bs{s},\bs{a})$ is $n R_{\max}$ while when only active arms yield rewards, the maximum value of $\bs{r}(\bs{s},\bs{a})$ is $m R_{\max}$. This leads to different bounds on $J_\theta$ and $\bs{V}_\theta$.
	\end{remark}
	
	\begin{remark} \label{rem:6}
		An alternative bound on $\SPAN(\bs{V}_\theta)$ is presented in \cite{bartlett2009regal}. Let ${\bs T}^{\bs{s_1} \to \bs{s_2}}_\theta$ denote the expected number of steps to go from state~$\bs{s_1}$ to state~$\bs{s_2}$ under policy $\bs{\mu}_\theta$ for model~$\theta$. Define \( \bs{D}_\theta = \max_{\bs{s_1}, \bs{s_2}} {\bs T}^{\bs{s_1} \to \bs{s_2}}_\theta \) to be the \textit{one-way diameter}. Then, it is shown in \cite{bartlett2009regal} that $\SPAN(\bs{V}_\theta) \leq J_\theta \bs{D}_\theta$. We do not know of an easy way to characterize the dependence of $\bs{D}_\theta$ on the number $n$ of arms. That is why we consider an alternative bound on $\SPAN(\bs{V}_\theta)$.
	\end{remark}
	
	\subsection{Regret decomposition.}
	For the ease of notation, we simply use $\REGRET(T)$ instead of
	$\REGRET(T;{\tt RB-TSDE})$. We also use $(J_\star, \mu_\star, \bs{P}_\star, \bs{V}_\star)$ instead of $(J_{\theta_\star}, \mu_{\theta_\star}, \bs{P}_{\theta_\star}, \bs{V}_{\theta_\star})$
	and use $(J_k, \mu_k, \bs{P}_k, \bs{V}_k)$ instead of $(J_{\theta_k}, \mu_{\theta_k}, \bs{P}_{\theta_k}, \bs{V}_{\theta_k})$.
	Rearranging terms in Bellman equation~\eqref{eqn:Bellman} and adding and subtracting $\bs{V}_k(\bs{s}_{t+1})$, we get:
	\begin{align}
		\bs{r}(\bs{s}_t, \bs{a}_t) & = J_k + \bs{V}_k(\bs{s}_t) - \bs{V}_k(\bs{s}_{t+1}) + \bs{V}_k(\bs{s}_{t+1}) \notag
		\\
		& \quad - \bigl[ \bs{P}_k \bs{V}_k \bigr](\bs{s_t}).
		\label{eq:modified-bellman}
	\end{align}
	
	Let $K_T$ denote the number of episodes until horizon~$T$.
	Substituting~\eqref{eq:modified-bellman} in~\eqref{eqn:regret}, we get:
	\begin{align}
		\hskip 1em & \hskip -1em
		\displaybreak[2]
		\REGRET(T) =
		\underbrace{ \EXP\biggl[ T J_\star -  \sum_{k=1}^{K_T} T_k J_k \biggr]}
		_{\text{regret due to sampling error $\eqqcolon \REGRET_0(T)$}}
		\notag \\
		\displaybreak[2]
		&\hskip 4em + 
		\underbrace{ \EXP\biggl[\sum_{k=1}^{K_T} \sum_{t=t_k}^{t_{k+1} - 1} 
			\bs{V}_k(\bs{S}_{t+1}) - \bs{V}_k(\bs{S}_{t}) \biggr]}
		_{\text{regret due to time-varying policy $\eqqcolon \REGRET_1(T)$}}
		\notag \\
		&\hskip 4em + 
		\underbrace{ \EXP\biggl[\sum_{k=1}^{K_T} \sum_{t=t_k}^{t_{k+1} - 1} 
			\bigl[ \bs{P}_k \bs{V}_k \bigr](\bs{S}_{t}) - \bs{V}_k(\bs{S}_{t+1}) \biggr]}
		_{\text{regret due to model mismatch $\eqqcolon \REGRET_2(T)$}}.
		\label{eq:decompose}
	\end{align}
	
	\subsection{Bounding individual terms.} 
	Each term of~\eqref{eq:decompose} is bounded as follows. 
	%Note Lemmas~\ref{lemma:R0_bound}, \ref{lemma:R1_bound}, and \ref{lemma:R2_bound} hold under Assumptions~\ref{assump:indexable}--\ref{assmp:3}, while Lemma~\ref{lemma:R2_bound_tight} holds under Assumptions~\ref{assump:indexable}--\ref{assump:lipschitz}.
	\begin{lemma} \label{lemma:R0_bound}
		Under Assumptions~\ref{assump:indexable}--\ref{assmp:3}, we have
		\begin{enumerate}
			\item $\displaystyle \REGRET_0(T) \leq 2 \alpha R_{\max} \Sqrt{\bar {\mathsf{S}}_n T \log T}$.
			\item $\displaystyle \REGRET_1(T) \leq  4 \frac{\alpha R_{\max}}{1-\lambda^*} \Sqrt{ \bar {\mathsf{S}}_n T \log T}$.
			\item $\displaystyle \REGRET_2(T) \leq  12\sqrt{2} \frac{\alpha R_{\max}}{1 - \lambda^*} \bigl( n + \bar {\mathsf{S}}_n \sqrt{T \log T} \bigr)$.
		\end{enumerate}
	\end{lemma}
	See the appendix for the proof steps.
	
	We can obtain an alternative bound on $\REGRET_2(T)$ under Assumption~\ref{assump:lipschitz}.
	
	\begin{lemma} \label{lemma:R2_bound_tight}
		Under Assumptions~\ref{assump:indexable}, \ref{assump:acoe}, and \ref{assump:lipschitz}, we have
		\[ \displaystyle \REGRET_2(T) \leq 12 \bar{\mathsf{S}}_n \mathsf{D}_{\max} \bs{L}_v \sqrt{K T \log (K' T)} \]
		where $K$ and $K'$ are positive constants that do not depend on $n$ and $T$.
	\end{lemma}
	See the appendix for the proof steps.
	
	\begin{remark} \label{rem:7}
		Under Assumption~\ref{assump:arm}, which will be described later, we can establish a tighter bound on $\SPAN(\bs{V}_\theta)$. In particular, Lipschitz continuity of $\bs{V}_\theta$ implies that $\SPAN(\bs{V}_\theta) \le \bs{L}_v \DIAM(\bs{\mathcal{S}})$. This can give us a tighter bound on the term $\REGRET_1(T)$, but this tighter bound does not help us in reducing the overall regret.
	\end{remark}
	
	\subsection{Obtaining the final bound.} \label{subsec:final-bound}
	\begin{proof}[Proof of Theorem~\ref{thm:main}]
		From Eq.~\eqref{eq:decompose} and Lemma~\ref{lemma:R0_bound}, we get
		\begin{align*}
			\REGRET(T) & \leq 2 \alpha R_{\max} \Sqrt{\bar {\mathsf{S}}_n T \log T} 
			+  4 \frac{\alpha R_{\max}}{1-\lambda^*} \Sqrt{ \bar {\mathsf{S}}_n T \log T} \notag \\
			& + 12\sqrt{2} \frac{\alpha R_{\max}}{1 - \lambda^*} \bigl( n + \bar {\mathsf{S}}_n \sqrt{T \log T} \bigr).
		\end{align*}
		By definition, $\lambda^* < 1$. Then,
		\begin{align*}
			\REGRET(T) & < 6 \frac{\alpha R_{\max}}{1-\lambda^*} \sqrt{\bar {\mathsf{S}}_n T \log T} + 24 \sqrt{2} \frac{\alpha R_{\max}}{1-\lambda^*} \bar {\mathsf{S}}_n \sqrt{T \log T} \\
			& < (6 + 24\sqrt{2}) \frac{\alpha R_{\max} }{1-\lambda^*} \bar {\mathsf{S}}_n \sqrt{T \log T} \\
			&= 40 \frac{\alpha R_{\max}}{1-\lambda^*} \bar {\mathsf{S}}_n \sqrt{T \log T}.
		\end{align*}
		This completes the proof of Theorem~\ref{thm:main}.
	\end{proof}
	
	\begin{proof}[Proof of Theorem~\ref{thm:lipschitz}]
		From Eq.~\eqref{eq:decompose}, Lemma~\ref{lemma:R0_bound}, and \ref{lemma:R2_bound_tight}, we get
		\begin{align*}
			\REGRET(T) & \leq 2 \alpha R_{\max} \Sqrt{\bar {\mathsf{S}}_n T \log T} 
			+  4 \frac{\alpha R_{\max}}{1-\lambda^*} \Sqrt{ \bar {\mathsf{S}}_n T \log T} \notag \\
			& \qquad + 12 \bar{\mathsf{S}}_n \mathsf{D}_{\max} \bs{L}_v \sqrt{K T \log (K' T)} \notag \\
			& < 12 \frac{\alpha R_{\max}}{1-\lambda^*} \sqrt{\bar {\mathsf{S}}_n T \log T} \notag \\
			& \qquad + 12 \bar{\mathsf{S}}_n \mathsf{D}_{\max} \bs{L}_v \sqrt{K T \log (K' T)}.
		\end{align*}
		Let $\bar{R} = \max\Bigl\{ R_{\max}/(1-\lambda^*), \mathsf{D}_{\max} \bs{L}_v \Bigr\}$. Then, we have
		\begin{align*}
			\REGRET(T) & < 12 \max\{ \alpha \sqrt{\bar {\mathsf{S}}_n}, \bar {\mathsf{S}}_n \} \bar{R} \sqrt{K T \log (K' T)}.
		\end{align*}
		This completes the proof of Theorem~\ref{thm:lipschitz}.
	\end{proof}
	
	\section{Numerical examples} \label{sec:numerical}
	In this section, we demonstrate the empirical performance of \texttt{RB-TSDE}. In particular, 
	we consider two environments, one for Model~A (a model for
	machine maintenance) and one for Model~B (a model for
	link scheduling). For both environments, we consider multiple experiments and plot the regret as a function of time and as a function of number of arms. Our results illustrate that the regret does indeed scale according to our theoretical results. We also compare the results with the empirical performance of \texttt{QWI}, which is a Q-learning algorithm for RBs proposed in~\cite{borkar2018learning,avrachenkov2020whittle,robledo2022qwi}. Note that only the algorithm proposed in~\cite{robledo2022qwi} is called \texttt{QWI}, but the algorithms of~\cite{borkar2018learning,avrachenkov2020whittle} are conceptually similar, so we collectively call them \texttt{QWI}.
	
	\subsection{Environments}
	
	We start with a description of the two environments.
	
	\subsubsection{Environment~A} 
	
	We consider a machine maintenance model where a single repairman is
	responsible for the maintenance of a set of machines, which deteriorate over
	time. Each machine has multiple deterioration states sorted from
	\emph{pristine} to \emph{ruined}. There is a cost associated with running the
	machine and the cost is non-decreasing function of the state. If the machine
	is left un-monitored, then the state of the machine stochastically
	deteriorates over time. The repairman may visit one of the machines and
	replace it with a new machine at a fixed cost. The objective is to determine a
	scheduling policy to minimize the expected discounted cost over time.
	
	We model the above environment as an instance of Model~A. In
	particular, we consider $n$ arms, where $n \in \{10, 20, \dots, 80\}$
	where $m=1$ arm can be activated at each time. The state space of each arm
	is of size $\mathsf{S} = 10$. Under $a = 1$, the state of the arm is reset to $1$ (and
	this fact is known to the learner). Under $a = 0$, the transition matrix is
	stochastic monotone and chosen as described in \cite[Appendix 1.2]{akbarzadeh2022conditions}. 
	The transitions under the passive action are unknown to the learner. 
	The per-step reward function are given by $r^i(s,0) = (\mathsf{S} -1)^2 - (s-1)^2$, 
	$r^i(s,1) = 0.5(\mathsf{S} - 1)^2$ for all $i \in [n]$ and $s \in [\mathsf{S}]$. 
	
	\subsubsection{Environment~B}

	We consider a link scheduling problem where there are $n$ users who can
	communicate over a shared communication link. Each user has a queue, where
	packets arrive according to an unknown i.i.d.\ process. At each time, a
	controller may schedule one of the users and transmit all its packets over the
	channel. The users which are not scheduled incur a holding cost which is equal
	to the square of the number of packets in the queue.

	We model the above environment as an instance of Model~B. In particular, we
	consider $n$ arms, where $n \in \{10, 20, \dots, 80\}$ where $m = n-1$ arm can
	be activated at each time. The state of each arm is of size $\mathsf{S} = 10$.
	Under $a = 1$, the transition matrix is upper-triangular and chosen as described in \cite[${\cal P}_1(p)$ of
	Appendix A]{akbarzadeh2022two} where $p$ is set to be different for each arm,
	linearly ranged from $0.05$ to $0.95$. The transitions under the active action
	are unknown to the learner. Under $a = 0$, the state of the arm is reset to
	$1$ (and this fact is known to the learner). The per-step reward function are
	given by $r^i(s,0) = 0$, $r^i(s,1) = (\mathsf{S} -1)^2 - (s-1)^2$ for all $i \in [n]$ and $s \in [\mathsf{S}]$.

	\subsection{Algorithms}
	
	We compare the performance of two algorithms.
	
	\subsubsection{\texttt{RB-TSDE}} We consider the \texttt{RB-TSDE} algorithm described in Algorithm~\ref{Alg:RB-TSDE}. We initialize the algorithm with uninformed Dirichlet prior on the unknown parameters and update the posterior according to the conjugate posterior for Dirichlet priors. 
	
	\subsubsection{\texttt{QWI} } We also consider \texttt{QWI}, which is a Q-learning algorithm for RBs proposed in~\cite{borkar2018learning,avrachenkov2020whittle,robledo2022qwi} as a baseline. The algorithm has two learning rates. As recommended in \cite[Eq. (17)]{avrachenkov2020whittle} we pick the step-size sequence which has a good performance by setting parameters $C$ and $C'$ of \texttt{QWI} as $C = 0.03$ and $C' = 0.01$, where the numerical values were obtained by running a hyper-parameter search. 
	
	\begin{figure*}[t!]
		\centering
		\begin{subfigure}[h]{0.225\textwidth}
			\centering
			\includegraphics[width=0.99\linewidth]{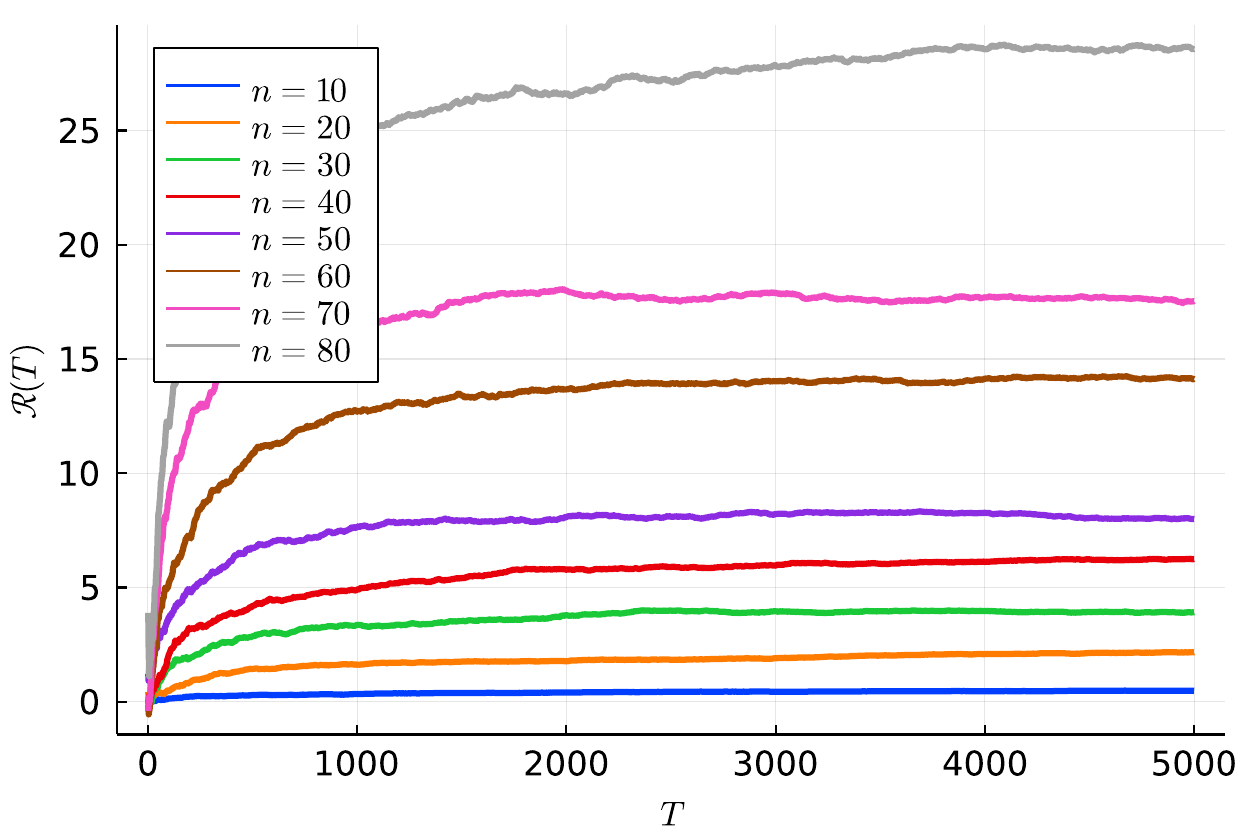}
			\caption[]%
			{RB-TSDE: {$\REGRET(T) / \sqrt{T}$ vs. $T$ }}    
			\label{fig:reg-a1}
		\end{subfigure}
		\quad
		\begin{subfigure}[h]{0.225\textwidth}  
			\centering 
			\includegraphics[width=0.99\linewidth]{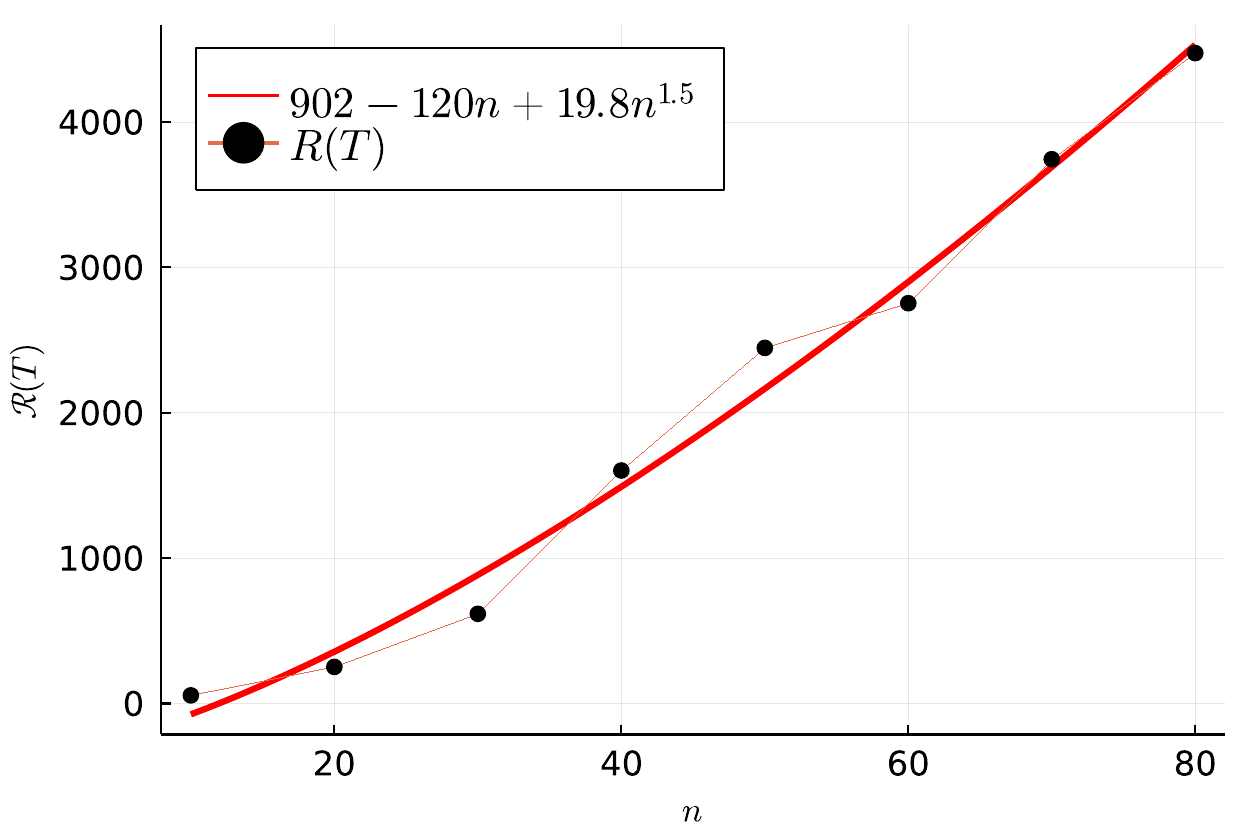}
			\caption[]%
			{{$\REGRET(T)$ vs. $n$}}    
			\label{fig:reg-a2}
		\end{subfigure}
		\quad
		\begin{subfigure}[h]{0.225\textwidth}   
			\centering 
			\includegraphics[width=0.99\linewidth]{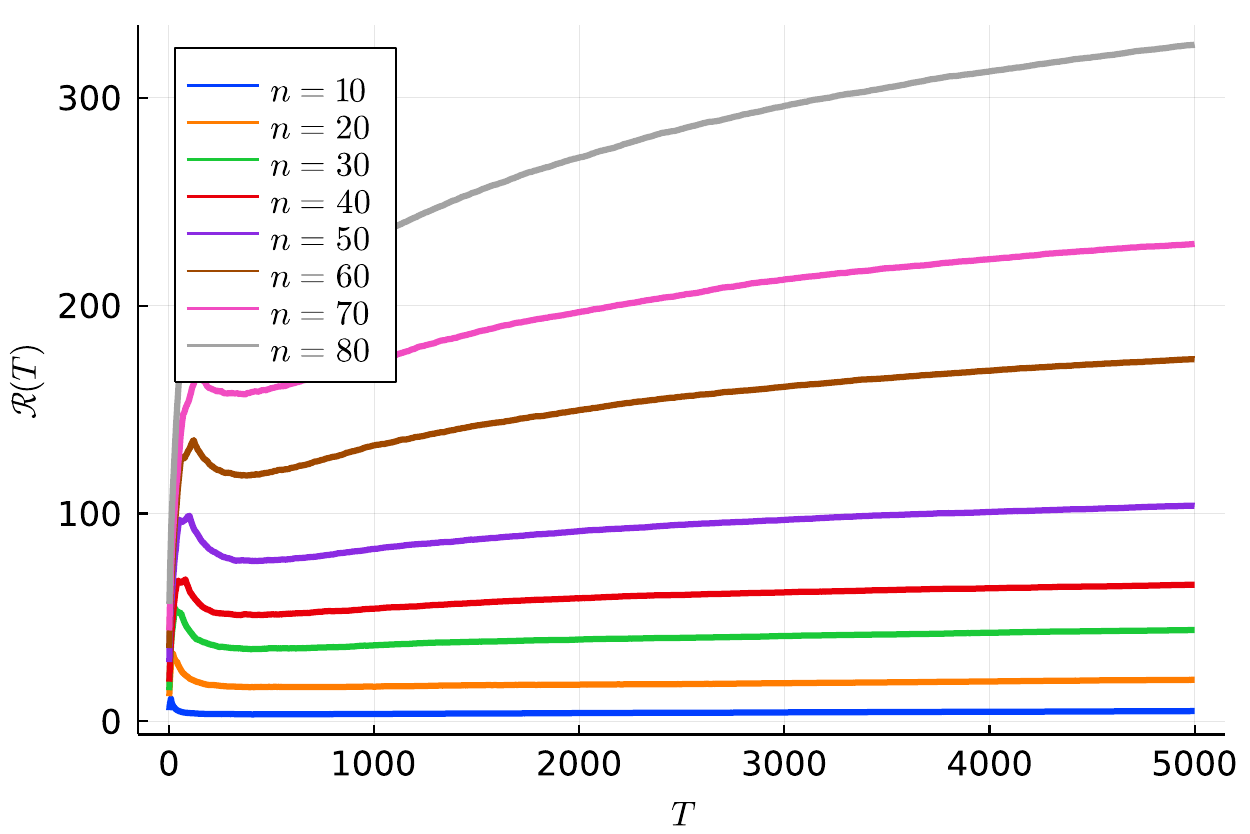}
			\caption[]%
			{{$\REGRET(T) / \sqrt{T}$ vs. $T$}}    
			\label{fig:reg-a3}
		\end{subfigure}
		\begin{subfigure}[h]{0.225\textwidth}
			\centering
			\includegraphics[width=0.99\linewidth]{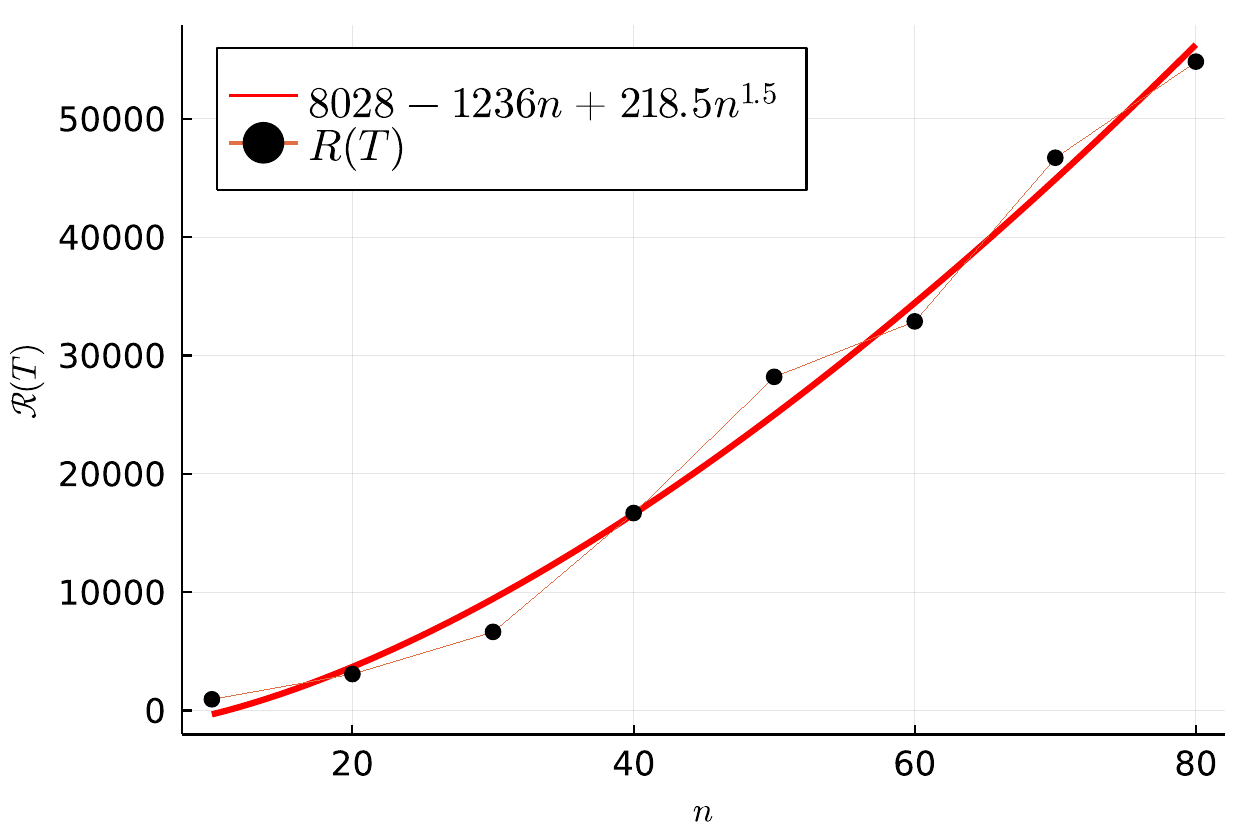}
			\caption[]%
			{{$\REGRET(T)$ vs. $n$}}    
			\label{fig:reg-aq1}
		\end{subfigure}
		\caption[]
		{Regret analysis of \texttt{RB-TSDE} and \texttt{QWI} for Environment A. Note that \texttt{RB-TSDE} has an order of magnitude better regret than \texttt{QWI}.} 
		\label{fig:reg-AT}
		\vskip \baselineskip
		\begin{subfigure}[h]{0.225\textwidth}
			\centering
			\includegraphics[width=0.99\linewidth]{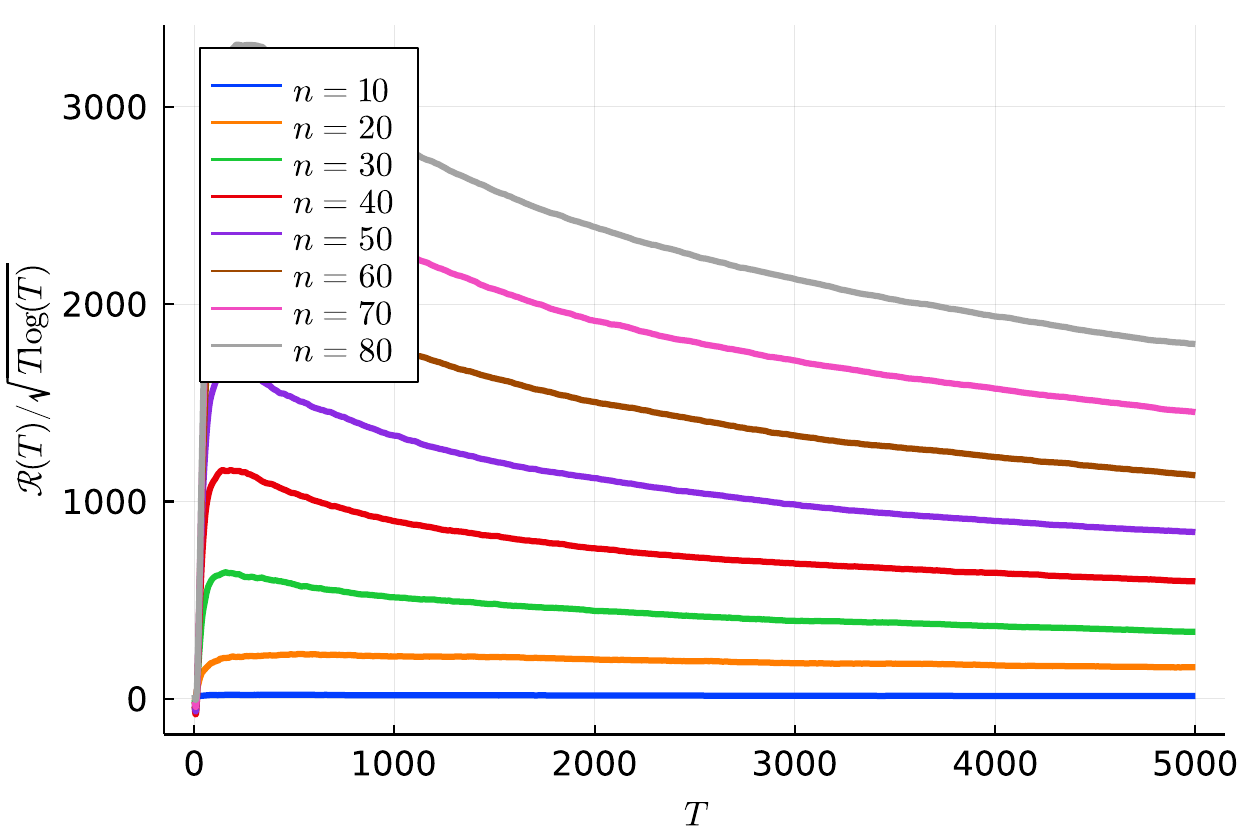}
			\caption[]%
			{{$\REGRET(T) / \sqrt{T}$ vs. $T$}}    
			\label{fig:reg-b1}
		\end{subfigure}
		\quad
		\begin{subfigure}[h]{0.225\textwidth}  
			\centering 
			\includegraphics[width=0.99\linewidth]{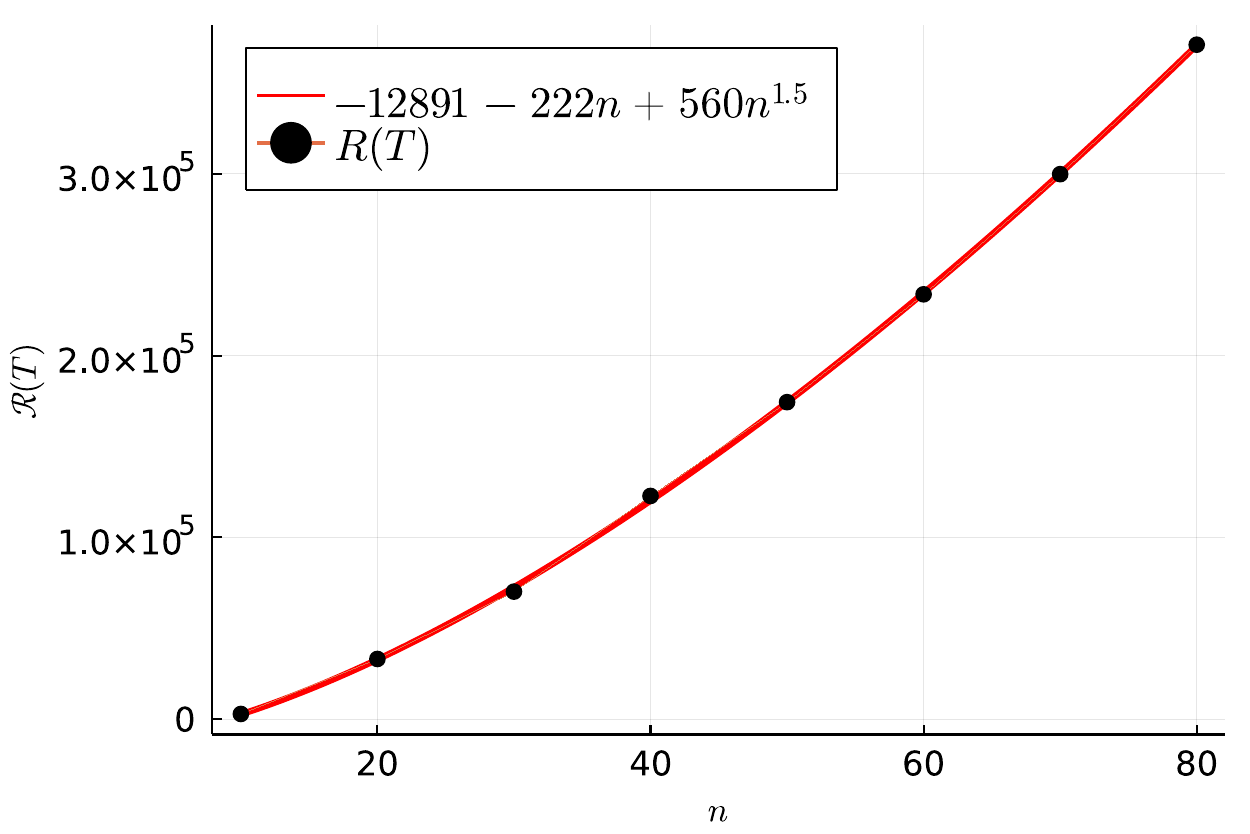}
			\caption[]%
			{{$\REGRET(T)$ vs. $n$}}    
			\label{fig:reg-b2}
		\end{subfigure}
		\quad
		\begin{subfigure}[h]{0.225\textwidth}   
			\centering 
			\includegraphics[width=0.99\linewidth]{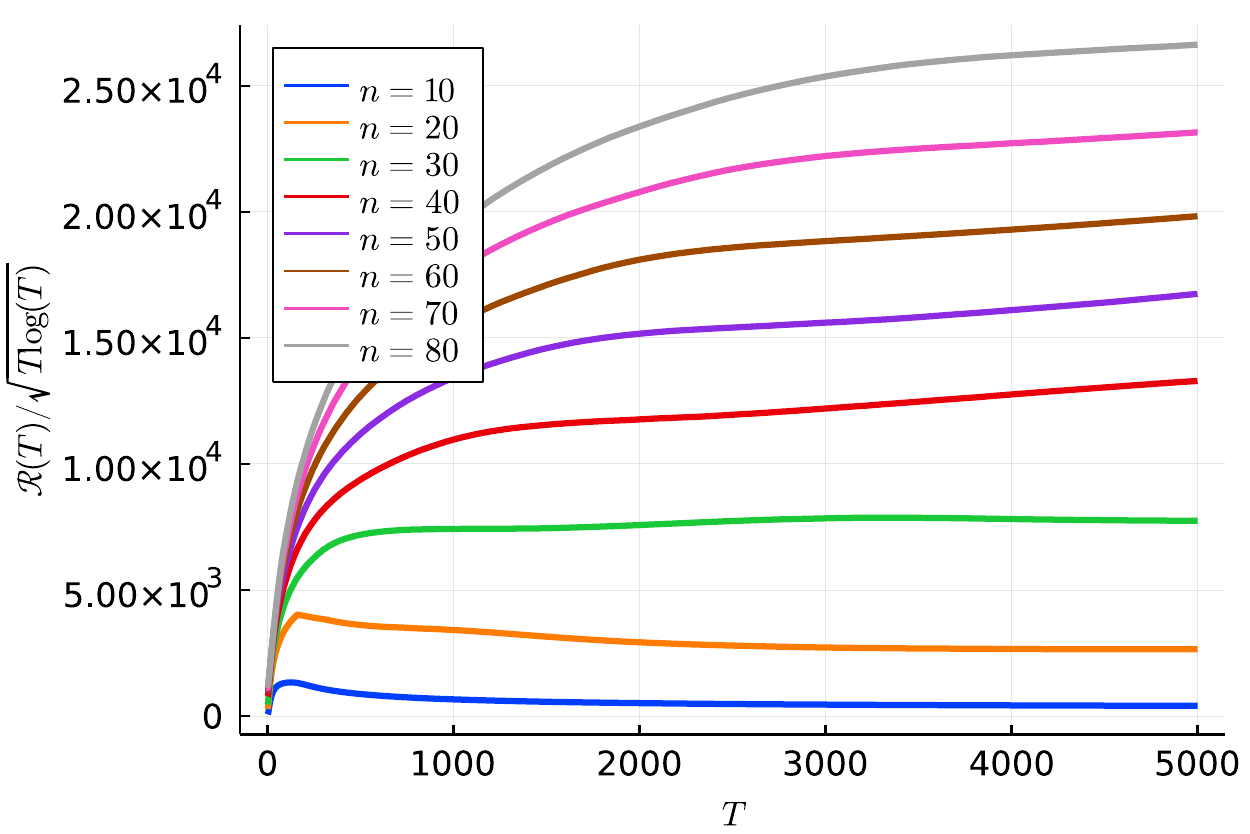}
			\caption[]%
			{{$\REGRET(T) / \sqrt{T}$ vs. $T$}}    
			\label{fig:reg-b3}
		\end{subfigure}
		\begin{subfigure}[h]{0.225\textwidth}
			\centering
			\includegraphics[width=0.99\linewidth]{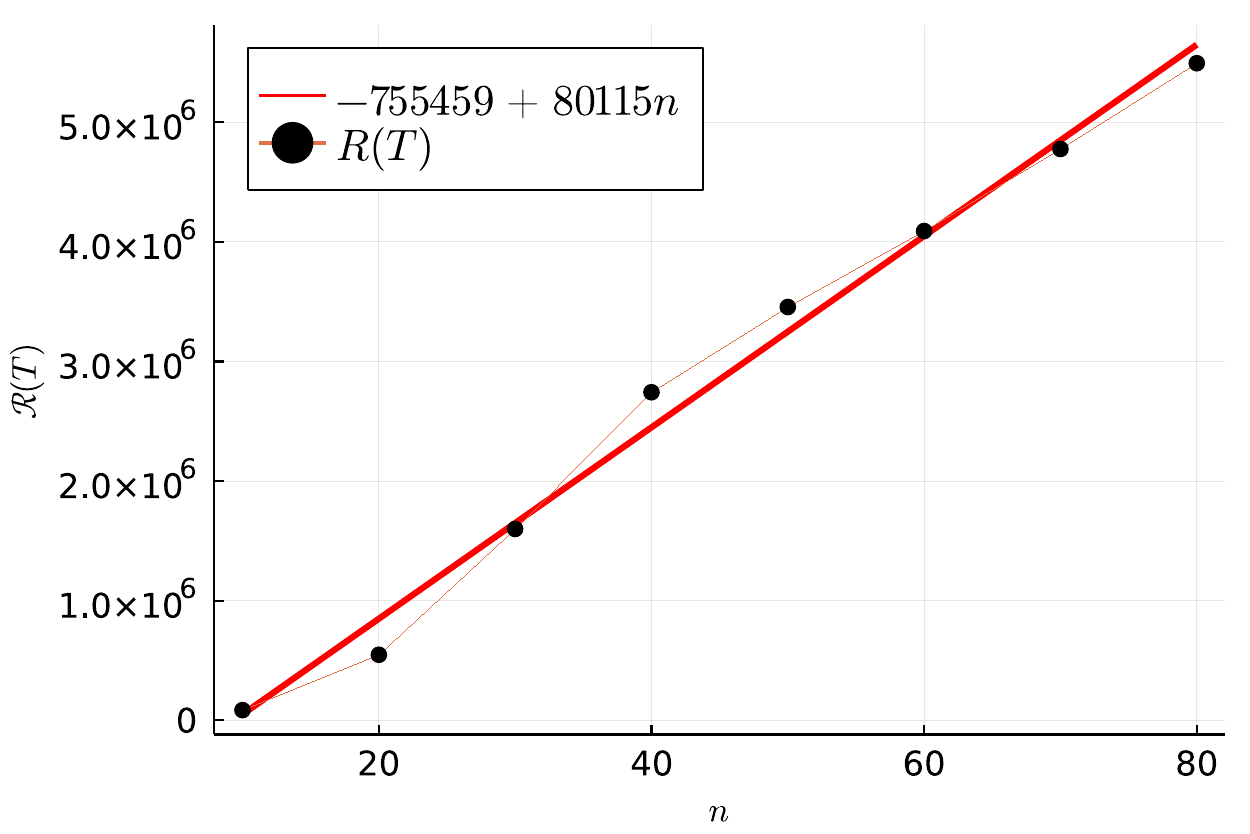}
			\caption[]%
			{{$\REGRET(T)$ vs. $n$}}    
			\label{fig:reg-bq1}
		\end{subfigure}
		\caption[]
		{Regret analysis of \texttt{RB-TSDE} and \texttt{QWI} method for Environment B. Note that \texttt{RB-TSDE} has an order of magnitude lower regret than \texttt{QWI}.}
		\label{fig:reg-BQ}
	\end{figure*}
	
	\subsection{Experimental Results} \label{subsec:exp}
	In our experiments, we pick a horizon of $T = 5,000$ and compute the Bayesian regret averaged over $250$ sample paths. We repeat the experiment for $n \in \{10, 20, \dots, 80 \}$. For each environment, we plot four curves: (a)~plot of $\mathcal{R}(T)/\sqrt{T}$ vs $T$ for \texttt{RB-TSDE}; (b)~plot of $\mathcal{R}(T)$ at $T=5,000$ vs $n$ for \texttt{RB-TSDE}; (c)~plot of $\mathcal{R}(T)/\sqrt{T}$ vs $T$ for \texttt{QWI}; and (d)~plot of $\mathcal{R}(T)$ at $T=5,000$ vs $n$ for \texttt{QWI}. For plots (b) and (d), we also fit the points with a parametric curve of the form $p_0 + p_1 n + p_2 n^{1.5}$ for Environment A and $p_0 + p_1 n$ for Environment B to obtain the scaling with number of arms. 
	
	The plots for Environment~A is shown in Fig.~\ref{fig:reg-AT}. The sub-plot~(a) shows that the regret essentially scales as $\sqrt{T}$ with time. The sub-plots~(b) show that the regret scales as $n^{1.5}$ with the number of arms. Thus, the results are consistent with regret bounds of Theorem~\ref{thm:lipschitz}.
	
	The plots for Environment~B is shown in Fig.~\ref{fig:reg-BQ}. The behavior of sub-plots~(a) and~(c) is the same as for Environment~A. Note that for larger values of~$T$, the plot of $\mathcal{R}(T)/\sqrt{T}$ has not yet converged to a straight line, but it is clear that these curves are upper bounded by a constant. The sub-plot~(b) shows that the regret scales linearly with the number of arms. 
	%Note that in this environment, $m$ is not changing with $n$. 
	Note that in this case since $m = n-1$, Theorem~\ref{thm:main} suggests that the regret should scale as $n^{1.5}$ (see Remark~\ref{rem:scaling}) and the result is consistent with the theorem.
	%The sub-plots~(b) show that the regret scales approximately as $\sqrt{T}$ with time. 
	
	Note that even though no regret analysis of the \texttt{QWI} algorithm for RBs is presented in~\cite{borkar2018learning,avrachenkov2020whittle,robledo2022qwi}, the above experiment suggests that empirically the performance of \texttt{QWI} has similar features as \texttt{RB-TSDE}. However, unlike the \texttt{QWI}, \texttt{RB-TSDE} does not require any hyper-parameter tuning. Moreover, \texttt{RB-TSDE} has an order of magnitude lower regret than \texttt{QWI}.
	
	%%%%%%%%%%%%%%%%%%%%%%%%%%%%%%%%%%%%%%%%%%%%%%%%%%%%%%%%%%%%%%%%%%%%%%%%%%%%%%%%%%%%%%%%%%%%%%%%%%%%%%%%%%%%%%%%%%%%%%%%%%%%%%%%%%%%%%%
	%%%%%%%%%%%%%%%%%%%%%%%%%%%%%%%%%%%%%%%%%%%%%%%%%%%%%%%%%%%%%%%%%%%%%%%%%%%%%%%%%%%%%%%%%%%%%%%%%%%%%%%%%%%%%%%%%%%%%%%%%%%%%%%%%%%%%%%
	%%%%%%%%%%%%%%%%%%%%%%%%%%%%%%%%%%%%%%%%%%%%%%%%%%%%%%%%%%%%%%%%%%%%%%%%%%%%%%%%%%%%%%%%%%%%%%%%%%%%%%%%%%%%%%%%%%%%%%%%%%%%%%%%%%%%%%%
	
	\section{Discussion} \label{sec:discussion}
	
	\subsection{A relaxation of Assumption~\ref{assmp:3}}\label{sec:relax}
	
	Assumption~\ref{assmp:3} can be relaxed as follows. 
	
	\begin{assumption} \label{assmp:5}
		For every~$\theta$, there exists a positive integer $\tau^*$ and a real $\lambda^* \in (0, 1)$ such that $\lambda_{P^{\tau^*}_\theta} \leq \lambda^*$.
	\end{assumption}
	
	Based on this assumption, the result of Lemmas~\ref{lem:bound} changes as follows.
	
	\begin{lemma}\label{lem:bound5}
		Under Assumption~\ref{assmp:5}, for any $\theta \in \Theta$, $\SPAN(\bs{V}_\theta) \le 2 \tau^* \alpha R_{\max}/(1 - \lambda^*)$.
	\end{lemma}
	Consequently, we have the following changes.
	\begin{lemma} \label{lemma:R0_bound5}
		Under Assumptions~\ref{assump:indexable}, \ref{assump:acoe}, and \ref{assmp:5}, we have
		\begin{enumerate}
			\item $\displaystyle \REGRET_1(T) \leq  4 \frac{\tau^* \alpha R_{\max}}{1-\lambda^*} \Sqrt{ \bar {\mathsf{S}}_n T \log T}$.
			\item $\displaystyle \REGRET_2(T) \leq  12\sqrt{2} \frac{\tau^* \alpha R_{\max}}{1 - \lambda^*} \bigl( n + \bar {\mathsf{S}}_n \sqrt{T \log T} \bigr)$.
		\end{enumerate}
	\end{lemma}
	
	\begin{theorem}\label{thm:main5}
		Under Assumptions~\ref{assump:indexable}, \ref{assump:acoe}, and \ref{assmp:5}, the regret of {\tt RB-TSDE} is upper bounded as follows:
		\[
		\REGRET(T; {\tt RB-TSDE}) 
		< 40 \alpha \frac{R_{\max}}{1-\lambda^*} \bar {\mathsf{S}}_n \sqrt{T \log T}.
		\]
	\end{theorem}
	
	\begin{theorem}\label{thm:lipschitz5}
		Under Assumptions~\ref{assump:indexable}, \ref{assump:acoe}, \ref{assump:lipschitz}, and \ref{assmp:5} the regret of {\tt RB-TSDE} for both models is upper bounded as follows:
		\begin{multline*}
			\REGRET(T; {\tt RB-TSDE})
			< 12 \max\{ \alpha \sqrt{\bar {\mathsf{S}}_n}, \bar {\mathsf{S}}_n \} \\
			\max\Bigl\{ \frac{ \tau^* R_{\max}}{1-\lambda^*}, \mathsf{D}_{\max} \bs{L}_v \Bigr\} \sqrt{K T \log (T\max\{1, K'\})}.
		\end{multline*}
	\end{theorem}
	The proof steps of Lemmas~\ref{lem:bound5}, \ref{lemma:R0_bound5} are similar to the proof steps of Lemmas~\ref{lem:bound}, \ref{lemma:R0_bound} and the proof steps of Theorems~\ref{thm:main5} and \ref{thm:lipschitz5} are similar to the proof steps of Theorems~\ref{thm:main} and \ref{thm:lipschitz}. See \cite{akbarzadeh2022learning} for the details.
	
	\subsection{A set of sufficient conditions for Assumption~\ref{assump:lipschitz}}\label{sec:sufficient}
	
	\begin{assumption}\label{assump:arm}
		Suppose each arm~$i \in [n]$ is $(L^i_r, L^i_p)$ Lipschitz, i.e.,
		\begin{align*}
			L^i_r & = \sup_{s^i_{(1)}, s^i_{(2)}, a} \dfrac{|r^i(s^i_{(1)}, a) - r^i(s^i_{(2)}, a)|}{d^i(s^i_{(1)}, s^i_{(2)})}, \notag \\
			L^i_p & = \sup_{s^i_{(1)}, s^i_{(2)}, a} \dfrac{\mathcal{K}(P^i(\cdot|s^i_{(1)}, a), P^i(\cdot|s^i_{(2)}, a))}{d^i(s^i_{(1)}, s^i_{(2)})}
		\end{align*}
		where $L^i_r < \infty$ and $L^i_p < 1$.
	\end{assumption}
	
	\begin{assumption}\label{assump:optimal}
		For all $\theta \in \Theta$, the Whittle index policy is optimal.
	\end{assumption}
	Assumption~\ref{assump:optimal} is satisfied in some instances such as: (i)~the rested multi-armed bandit setup described in Remark~\ref{rem:2} where only one arm can be activated at a time (i.e., $m=1$) and yield reward, and arms that are not activated remain frozen (i.e., $P(s^i_{+} | s^i, 0) = \delta_{s^i}(s^i_{+})$, where $\delta_{s^i}$ is the Dirac delta measure centered at $s^i$)~\cite{gittins1979bandit}; (ii)~the number of arms are asymptotically large~\cite{weber1990index}; (iii)~certain queuing models~\cite{lott2000optimality}.
	
	Moreover, we assume that the product measure on $\bs{S}$ is $\bs{d}(\bs{s},\bs{s'}) = \sum_{i \in [n]} d^i(s^i, s^{\prime,i})$. 
	
	\begin{lemma}[{\cite[Lemma 2]{robustness2022sinha}}]\label{lemma:lipschitz}
		Under Assumption~\ref{assump:arm}, the MDP~$\langle\bs{\mathcal{S}}, \bs{\mathcal{A}(m)}, \bs{P}, \bs{R}\rangle$ is $(\max_{i \in [n]} L^i_r, \max_{i \in [n]} L^i_p)$-Lipschitz i.e., 
		\begin{align*}
			\max_{\substack{\bs{s},\bs{s'} \in \bs{S} \\ \bs{a} \in \bs{A}(m)}}
			\frac{ | \bs{r}(\bs{s}, \bs{a}) - \bs{r}(\bs{s}',\bs{a}) |}
			{\bs{d}(\bs{s},\bs{s'})}
			&\le \max_{i \in [n]} L^i_r,
			\\
			\max_{\substack{\bs{s},\bs{s'} \in \bs{S} \\ \bs{a} \in \bs{A}(m)}}
			\frac{ \mathcal{K}(\bs{P}(\cdot | \bs{s}, \bs{a}) - \bs{P}(\cdot | \bs{s}',\bs{a}))}
			{\bs{d}(\bs{s},\bs{s'})}
			&\le \max_{i \in [n]} L^i_p.
		\end{align*}
	\end{lemma}
	
	An immediate consequence of Lemma~\ref{lemma:lipschitz} is the following.
	\begin{lemma} \label{lemma:LV}
		Under Assumptions~\ref{assump:arm} and~\ref{assump:optimal}, $\bs{V}_\theta$ is Lipschitz with the Lipschitz constant bounded by
		\[
		\bs{L}_v \le (\max_{i \in [n]} L^i_r)/(1 - \max_{i \in [n]} L^i_p).
		\]
		Thus, Assumptions~\ref{assump:arm} and~\ref{assump:optimal} imply Assumption~\ref{assump:lipschitz}.
	\end{lemma}
	\begin{proof}
		The result follows from Lemma~\ref{lemma:lipschitz} and~\cite[Theorem 4.2]{hinderer2005lipschitz}.
	\end{proof}
	
	%Note if the optimal policy is implemented according to Sec.~\ref{subsec:optimal}, then Assumption~\ref{assump:optimal} is no longer required.

	\subsection{Regret with respect to the optimal policy} \label{subsec:optimal}
	
	We measure regret with respect to the Whittle index policy. For models where Assumption~\ref{assump:optimal} is satisfied, which will be described later, the Whittle index policy is an optimal policy. Even when the assumption is not satisfied, it is possible to generalize the results of this paper to identify the regret with respect to the optimal policy. 
	In particular, let  $\bs{\psi}^*_\theta$ denote the optimal policy for model $\theta \in \Theta$. Then, the Bayesian regret of a learning algorithm $\pi$ with respect to the optimal policy is 
	\begin{equation}\label{eqn:regret-opt} 
		\REGRET^*(T; {\bs \pi}) = \EXP^{\bs \pi}\biggl[ T {\bs J}(\bs{\psi}^*_\theta) - \sum_{t = 1}^T \bs{r}({\bs S}_t, {\bs A}_t) \biggr].
	\end{equation}
	Then, in principle, we can replace the distributed implementation presented in Algorithm~\ref{Alg:RB-TSDE} with a modified centralized implementation where the learner observes the state of all arms and maintains the posterior $\phi^i_t$ for all $i \in [n]$. At the beginning of each episode, the learner samples $\theta^i_{t_k}$ from $\phi^i_{t_k}$, computes the policy $\mu_{t_k}$, which is optimal for the sampled model $(\theta^1_{t_k}, \dots, \theta^n_{t_k})$, and plays $\mu_{t_k}$ for the rest of the episode. The regret of this variant will be the same as the bounds in Theorems~\ref{thm:main} and~\ref{thm:lipschitz}. However, we do not present such an analysis here because it makes the resulting algorithm impractical as computing the optimal policy is intractable when there are more than a few arms.

	\section{Conclusion} \label{sec:conclusion}
	
	In this paper, we present a Thompson-sampling based reinforcement learning algorithm for restless bandits. We show that the Bayesian regret of our algorithm with respect to an oracle that applies the Whittle index policy of the true model is either $\tildeO(nm \sqrt{T})$, $\tildeO(n^2 \sqrt{T})$, $\tildeO(n^{1.5} \sqrt{T})$ or $\tildeO(n \sqrt{T})$ depending on assumptions on the model. These are in contrast to naively using any standard RL algorithm, which will have a regret that scales exponentially in $n$. Our results are also applicable to the rested multi-armed bandit setting, where the Whittle index policy is the same as the Gittins index and is optimal. All in all, our results illustrate that a learning algorithm which leverages the structure of the model can significantly improve regret compared to model-agnostic algorithms. 
	
	% that an RL algorithm which leverages teh structure of the model can have a significantly better regret
	% constitute an example of a large-scale system where a specific feature of the model can be exploited to develop RL algorithms, which have more refined regret than model-agnostic algorithms.

	\bibliographystyle{IEEEtran}
	\bibliography{IEEEabrv,mybibfile}

\begin{thebibliography}{10}
\providecommand{\url}[1]{#1}
\csname url@samestyle\endcsname
\providecommand{\newblock}{\relax}
\providecommand{\bibinfo}[2]{#2}
\providecommand{\BIBentrySTDinterwordspacing}{\spaceskip=0pt\relax}
\providecommand{\BIBentryALTinterwordstretchfactor}{4}
\providecommand{\BIBentryALTinterwordspacing}{\spaceskip=\fontdimen2\font plus
\BIBentryALTinterwordstretchfactor\fontdimen3\font minus
  \fontdimen4\font\relax}
\providecommand{\BIBforeignlanguage}[2]{{%
\expandafter\ifx\csname l@#1\endcsname\relax
\typeout{** WARNING: IEEEtran.bst: No hyphenation pattern has been}%
\typeout{** loaded for the language `#1'. Using the pattern for}%
\typeout{** the default language instead.}%
\else
\language=\csname l@#1\endcsname
\fi
#2}}
\providecommand{\BIBdecl}{\relax}
\BIBdecl

\bibitem{ouyang2015downlink}
W.~Ouyang, A.~Eryilmaz, and N.~B. Shroff, ``Downlink scheduling over
  {Markovian} fading channels,'' \emph{{IEEE/ACM} Trans. Netw.}, vol.~24,
  no.~3, pp. 1801--1812, 2015.

\bibitem{borkar2017opportunistic}
V.~S. Borkar, G.~S. Kasbekar, S.~Pattathil, and P.~Y. Shetty, ``Opportunistic
  scheduling as restless bandits,'' \emph{{IEEE} Trans. Control Netw. Syst.},
  vol.~5, no.~4, pp. 1952--1961, 2018.

\bibitem{wang2019opportunistic}
K.~Wang, J.~Yu, L.~Chen, P.~Zhou, X.~Ge, and M.~Z. Win, ``Opportunistic
  scheduling revisited using restless bandits: Indexability and index policy,''
  \emph{{IEEE} Trans. Wireless Commun.}, vol.~18, no.~10, pp. 4997--5010, 2019.

\bibitem{ali2018sleeping}
S.~Ali, A.~Ferdowsi, W.~Saad, and N.~Rajatheva, ``Sleeping multi-armed bandits
  for fast uplink grant allocation in machine type communications,'' in
  \emph{IEEE Globecom}.\hskip 1em plus 0.5em minus 0.4em\relax IEEE, 2018, pp.
  1--6.

\bibitem{singh2022user}
S.~K. Singh, V.~S. Borkar, and G.~Kasbekar, ``User association in dense
  {mmWave} networks as restless bandits,'' \emph{{IEEE} Trans. Veh. Technol.},
  2022.

\bibitem{lott2000optimality}
C.~Lott and D.~Teneketzis, ``On the optimality of an index rule in multichannel
  allocation for single-hop mobile networks with multiple service classes,''
  \emph{Probab. Engrg. Inform. Sci.}, vol.~14, no.~3, pp. 259--297, 2000.

\bibitem{si2009distributed}
P.~Si, F.~R. Yu, H.~Ji, and V.~C. Leung, ``Distributed multisource transmission
  in wireless mobile peer-to-peer networks: a restless-bandit approach,''
  \emph{{IEEE} Trans. Veh. Technol.}, vol.~59, no.~1, pp. 420--430, 2009.

\bibitem{liu2010indexability}
K.~Liu and Q.~Zhao, ``Indexability of restless bandit problems and optimality
  of {Whittle} index for dynamic multichannel access,'' \emph{{IEEE} Trans.
  Inf. Theory}, vol.~56, no.~11, pp. 5547--5567, 2010.

\bibitem{nino2009restless}
J.~Nino-Mora, ``A restless bandit marginal productivity index for opportunistic
  spectrum access with sensing errors,'' in \emph{International Conference on
  Network Control and Optimization}.\hskip 1em plus 0.5em minus 0.4em\relax
  Springer, 2009, pp. 60--74.

\bibitem{tekin2011online}
C.~Tekin and M.~Liu, ``Online learning in opportunistic spectrum access: A
  restless bandit approach,'' in \emph{INFOCOM}.\hskip 1em plus 0.5em minus
  0.4em\relax IEEE, 2011, pp. 2462--2470.

\bibitem{wang2014adaptive}
Q.~Wang, M.~Liu, and J.~L. Mathieu, ``Adaptive demand response: Online learning
  of restless and controlled bandits,'' in \emph{Int. Conf. Smart Grid
  Commun.}\hskip 1em plus 0.5em minus 0.4em\relax IEEE, 2014, pp. 752--757.

\bibitem{abad2016near}
C.~Abad and G.~Iyengar, ``A near-optimal maintenance policy for automated {DR}
  devices,'' \emph{{IEEE} Trans. Smart Grid}, vol.~7, no.~3, pp. 1411--1419,
  2016.

\bibitem{le2008multi}
J.~Le~Ny, M.~Dahleh, and E.~Feron, ``Multi-{UAV} dynamic routing with partial
  observations using restless bandit allocation indices,'' in \emph{American
  Control Conference}.\hskip 1em plus 0.5em minus 0.4em\relax IEEE, 2008, pp.
  4220--4225.

\bibitem{dahiya2022scalable}
A.~Dahiya, N.~Akbarzadeh, A.~Mahajan, and S.~L. Smith, ``Scalable operator
  allocation for multirobot assistance: A restless bandit approach,''
  \emph{{IEEE} Trans. Control Netw. Syst.}, vol.~9, no.~3, pp. 1397--1408,
  2022.

\bibitem{papadimitriou1999complexity}
C.~H. Papadimitriou and J.~N. Tsitsiklis, ``The complexity of optimal queuing
  network control,'' \emph{Math. Operat. Res.}, vol.~24, no.~2, pp. 293--305,
  1999.

\bibitem{whittle1988restless}
P.~Whittle, ``Restless bandits: Activity allocation in a changing world,''
  \emph{J. Appl. Prob.}, vol.~25, no.~A, pp. 287--298, 1988.

\bibitem{gittins1979bandit}
J.~C. Gittins, ``Bandit processes and dynamic allocation indices,''
  \emph{Journal of the Royal Statistical Society. Series B}, pp. 148--177,
  1979.

\bibitem{weber1990index}
R.~R. Weber and G.~Weiss, ``On an index policy for restless bandits,'' \emph{J.
  Appl. Prob.}, vol.~27, no.~3, pp. 637--648, 1990.

\bibitem{glazebrook2002index}
K.~Glazebrook and H.~Mitchell, ``An index policy for a stochastic scheduling
  model with improving/deteriorating jobs,'' \emph{Naval Res. Logist.},
  vol.~49, no.~7, pp. 706--721, 2002.

\bibitem{glazebrook2005index}
K.~D. Glazebrook, H.~M. Mitchell, and P.~S. Ansell, ``Index policies for the
  maintenance of a collection of machines by a set of repairmen,'' \emph{Euro.
  J. Operat. Res.}, vol. 165, no.~1, pp. 267--284, 2005.

\bibitem{glazebrook2006some}
K.~D. Glazebrook, D.~Ruiz-Hernandez, and C.~Kirkbride, ``Some indexable
  families of restless bandit problems,'' \emph{Adv. Appl. Prob.}, vol.~38,
  no.~3, pp. 643--672, 2006.

\bibitem{ninomora2007}
J.~Niño-Mora, ``A {$(2/3)n^3$} fast-pivoting algorithm for the gittins index
  and optimal stopping of a {M}arkov {C}hain,'' \emph{INFORMS J. Comput.},
  vol.~19, no.~4, pp. 596--606, 2007.

\bibitem{ayesta2010modeling}
U.~Ayesta, M.~Erausquin, and P.~Jacko, ``A modeling framework for optimizing
  the flow-level scheduling with time-varying channels,'' \emph{Perform
  Evaluation}, vol.~67, no.~11, pp. 1014--1029, 2010.

\bibitem{akbarzadeh2022conditions}
N.~Akbarzadeh and A.~Mahajan, ``Conditions for indexability of restless bandits
  and an $\mathcal{O}\!\left(k^3\right)$ algorithm to compute whittle index,''
  \emph{Adv. in Appl. Probab.}, vol.~54, no.~4, p. 1164–1192, 2022.

\bibitem{meshram2017restless}
R.~Meshram, A.~Gopalan, and D.~Manjunath, ``Restless bandits that hide their
  hand and recommendation systems,'' in \emph{Int. Conf. Commun. Syst.
  Netw.}\hskip 1em plus 0.5em minus 0.4em\relax IEEE, 2017, pp. 206--213.

\bibitem{borkar2018learning}
V.~S. Borkar and K.~Chadha, ``A reinforcement learning algorithm for restless
  bandits,'' in \emph{Indian Control Conference}, 2018, pp. 89--94.

\bibitem{Fu2019Qlearning}
J.~Fu, Y.~Nazarathy, S.~Moka, and P.~G. Taylor, ``Towards {Q}-learning the
  {Whittle} index for restless bandits,'' in \emph{Australian New Zealand
  Control Conference (ANZCC)}, 2019, pp. 249--254.

\bibitem{avrachenkov2020whittle}
K.~Avrachenkov and V.~S. Borkar, ``Whittle index based {Q}-learning for
  restless bandits with average reward,'' \emph{arXiv preprint
  arXiv:2004.14427}, 2020.

\bibitem{robledo2022qwi}
F.~Robledo, V.~Borkar, U.~Ayesta, and K.~Avrachenkov, ``{QWI}: {Q}-learning
  with {Whittle} index,'' \emph{ACM SIGMETRICS Performance Evaluation Review},
  vol.~49, no.~2, pp. 47--50, 2022.

\bibitem{xiong2021learning}
G.~Xiong, R.~Singh, and J.~Li, ``Learning augmented index policy for optimal
  service placement at the network edge,'' \emph{arXiv preprint
  arXiv:2101.03641}, 2021.

\bibitem{tekin2012online}
C.~Tekin and M.~Liu, ``Online learning of rested and restless bandits,''
  \emph{{IEEE} Trans. Inf. Theory}, vol.~58, no.~8, pp. 5588--5611, 2012.

\bibitem{ortner2012regret}
R.~Ortner, D.~Ryabko, P.~Auer, and R.~Munos, ``Regret bounds for restless
  {M}arkov bandits,'' in \emph{Int. conf. algorithmic learn. theory}.\hskip 1em
  plus 0.5em minus 0.4em\relax Springer, 2012, pp. 214--228.

\bibitem{liu2013learning}
H.~Liu, K.~Liu, and Q.~Zhao, ``Learning in a changing world: Restless
  multiarmed bandit with unknown dynamics,'' \emph{{IEEE} Trans. Inf. Theory},
  vol.~59, no.~3, pp. 1902--1916, 2013.

\bibitem{jung2019thompson}
Y.~H. Jung, M.~Abeille, and A.~Tewari, ``Thompson sampling in non-episodic
  restless bandits,'' \emph{arXiv preprint arXiv:1910.05654}, 2019.

\bibitem{jung2019episodic}
Y.~H. Jung and A.~Tewari, ``Regret bounds for {Thompson} sampling in episodic
  restless bandit problems,'' in \emph{Neural Inf. Process. Syst.},
  vol.~32.\hskip 1em plus 0.5em minus 0.4em\relax Curran Associates, Inc.,
  2019.

\bibitem{gafni2020learning}
T.~Gafni and K.~Cohen, ``Learning in restless multi-armed bandits via adaptive
  arm sequencing rules,'' \emph{{IEEE} Trans. Autom. Control}, vol.~66, no.~10,
  pp. 5029--5036, Oct. 2020.

\bibitem{jaksch2010near}
T.~Jaksch, R.~Ortner, and P.~Auer, ``Near-optimal regret bounds for
  reinforcement learning,'' \emph{J Mach Learn Res}, vol.~11, no.~4, 2010.

\bibitem{bartlett2009regal}
P.~L. Bartlett and A.~Tewari, ``Regal: A regularization based algorithm for
  reinforcement learning in weakly communicating {MDPs},'' in \emph{Uncertain.
  Artif. Intell.}, Arlington, Virginia, USA, 2009, p. 35–42.

\bibitem{fruit2018efficient}
R.~Fruit, M.~Pirotta, A.~Lazaric, and R.~Ortner, ``Efficient
  bias-span-constrained exploration-exploitation in reinforcement learning,''
  in \emph{Int. Conf. Mach. Learn.}\hskip 1em plus 0.5em minus 0.4em\relax
  PMLR, 2018, pp. 1578--1586.

\bibitem{agrawal2017posterior}
S.~Agrawal and R.~Jia, ``Posterior sampling for reinforcement learning:
  worst-case regret bounds,'' in \emph{Neural Inf. Process. Syst.}, 2017, pp.
  1184--1194.

\bibitem{zhang2019regret}
Z.~Zhang and X.~Ji, ``Regret minimization for reinforcement learning by
  evaluating the optimal bias function,'' in \emph{Neural Inf. Process. Syst.},
  vol.~32.\hskip 1em plus 0.5em minus 0.4em\relax Curran Associates, Inc.,
  2019.

\bibitem{ouyang2017learning}
Y.~Ouyang, M.~Gagrani, A.~Nayyar, and R.~Jain, ``Learning unknown {M}arkov
  decision processes: A {T}hompson sampling approach,'' in \emph{Neural Inf.
  Process. Syst.}, vol.~30.\hskip 1em plus 0.5em minus 0.4em\relax Curran
  Associates, Inc., 2017.

\bibitem{beutler1985optimal}
F.~J. Beutler and K.~W. Ross, ``Optimal policies for controlled {Markov} chains
  with a constraint,'' \emph{J. Math. Anal. Appl.}, vol. 112, no.~1, pp.
  236--252, 1985.

\bibitem{avrachenkov2013congestion}
K.~Avrachenkov, U.~Ayesta, J.~Doncel, and P.~Jacko, ``Congestion control of
  {TCP} flows in internet routers by means of index policy,'' \emph{Comm Com
  Inf Sc}, vol.~57, no.~17, pp. 3463--3478, 2013.

\bibitem{wang2019whittle}
J.~{Wang}, X.~{Ren}, Y.~{Mo}, and L.~{Shi}, ``Whittle index policy for dynamic
  multichannel allocation in remote state estimation,'' \emph{{IEEE} Trans.
  Autom. Control}, vol.~65, no.~2, pp. 591--603, 2020.

\bibitem{nino2007dynamic}
J.~Ni{\~n}o-Mora, ``Dynamic priority allocation via restless bandit marginal
  productivity indices,'' \emph{TOP}, vol.~15, no.~2, pp. 161--198, 2007.

\bibitem{rusmevichientong2010linearly}
P.~Rusmevichientong and J.~N. Tsitsiklis, ``Linearly parameterized bandits,''
  \emph{Math. Oper. Res.}, vol.~35, no.~2, pp. 395--411, 2010.

\bibitem{agrawal2013thompson}
S.~Agrawal and N.~Goyal, ``Thompson sampling for contextual bandits with linear
  payoffs,'' in \emph{Int. Conf. Mach. Learn.}\hskip 1em plus 0.5em minus
  0.4em\relax PMLR, 2013, pp. 127--135.

\bibitem{russo2014learning}
D.~Russo and B.~Van~Roy, ``Learning to optimize via posterior sampling,''
  \emph{Math. Oper. Res.}, vol.~39, no.~4, pp. 1221--1243, 2014.

\bibitem{kaufmann2011efficiency}
E.~Kaufmann, O.~Capp{\'e}, and A.~Garivier, ``On the efficiency of bayesian
  bandit algorithms from a frequentist point of view,'' in \emph{Neural Inf.
  Process. Syst.}, 2011.

\bibitem{kaufmann2014analysis}
E.~Kaufmann, ``Analysis of bayesian and frequentist strategies for sequential
  resource allocation,'' Ph.D. dissertation, T{\'e}l{\'e}com ParisTech, 2014.

\bibitem{arapostathis1993}
A.~Arapostathis, V.~Borkar, E.~Fern{\'a}ndez-Gaucherand, M.~Ghosh, and
  S.~Marcus, ``Discrete-time controlled {M}arkov processes with average cost
  criterion: A survey,'' \emph{SIAM J. Control Optim}, vol.~31, pp. 282--334,
  1993.

\bibitem{morton1977discounting}
T.~E. Morton and W.~E. Wecker, ``Discounting, ergodicity and convergence for
  {Markov} decision processes,'' \emph{Manage Sci}, vol.~23, no.~8, pp.
  890--900, 1977.

\bibitem{puterman2014markov}
M.~L. Puterman, \emph{{Markov} decision processes: discrete stochastic dynamic
  programming}.\hskip 1em plus 0.5em minus 0.4em\relax John Wiley \& Sons,
  2014.

\bibitem{akbarzadeh2022two}
N.~Akbarzadeh and A.~Mahajan, ``Two families of indexable partially observable
  restless bandits and {Whittle} index computation,'' \emph{arXiv preprint
  arXiv:2104.05151}, 2022.

\bibitem{akbarzadeh2022learning}
------, ``On learning {Whittle} index policy for restless bandits with scalable
  regret,'' \emph{arXiv preprint arXiv:2202.03463}, 2022.

\bibitem{robustness2022sinha}
A.~Sinha and A.~Mahajan, ``Robustness of {Whittle} index policy to model
  approximation,'' \emph{preprint}, 2022.

\bibitem{hinderer2005lipschitz}
K.~Hinderer, ``Lipschitz continuity of value functions in {M}arkovian decision
  processes,'' \emph{Math. Methods Oper. Res.}, vol.~62, no.~1, pp. 3--22,
  2005.

\bibitem{weissman2003inequalities}
T.~Weissman, E.~Ordentlich, G.~Seroussi, S.~Verdu, and M.~J. Weinberger,
  ``Inequalities for the l1 deviation of the empirical distribution,''
  \emph{Hewlett-Packard Labs, Tech. Rep}, 2003.

\bibitem{fournier2015rate}
N.~Fournier and A.~Guillin, ``On the rate of convergence in {Wasserstein}
  distance of the empirical measure,'' \emph{Probab. Theory Related Fields},
  vol. 162, no.~3, pp. 707--738, 2015.

\bibitem{akbarzadeh2019restless}
N.~Akbarzadeh and A.~Mahajan, ``Restless bandits with controlled restarts:
  Indexability and computation of {Whittle} index,'' in \emph{Conf. Decision
  Control}, 2019, pp. 7294--7300.

\end{thebibliography}
	
	\onecolumn
	\appendix
	
	\section{Proof of Regret bounds (Theorems~\ref{thm:main} and \ref{thm:lipschitz})} \label{app:proof}
	
	\subsection{Bound on the number of episodes} \label{app:KT-proof}
	We first start by establishing a bound on the number of episodes. The proof has the same structure as the standard method to establish a bound on the number of macro episodes~\cite[Lemma~2]{ouyang2017learning} but certain steps are different due to the difference in the algorithm.
	
	\begin{lemma}\label{lemma:KT_bound_general}
		The number of episodes $K_T$ is bounded as follows:
		\[
		K_T \leq 2 \Sqrt{\bar {\mathsf{S}}_n T \log (T)}.
		\]
	\end{lemma}
	\begin{proof}
		Define macro episodes with start times $t_{n_l}$, $l = 1, 2, \ldots$ with $t_{n_1} = t_1$ and \[ t_{n_{l+1}} = \min \{ t_k > t_{n_l}: N^i_{t_k}(s^i, a^i) > 2 N^i_{t_{k-1}}(s^i, a^i) \text{ for some } (i, s^i, a^i) \}. \]
		Let $\gamma$ be the number of macro episodes until time $T$ and define $n_{(\gamma+1)} = K_T + 1$. The rest of the proof is the same as \cite[Eq. (8) in proof of Lemma 1]{ouyang2017learning} by which we get $K_T \leq \sqrt{2 \gamma T}$.
		
		For each arm-state-action tuple, define
		\[ \gamma^i(s^i, a^i) = |\{ k \leq K_T | N^i_{t_k}(s^i, a^i) > 2 N^i_{t_{k-1}}(s^i, a^i) \}|. \]
		As a result $\gamma^i(s^i, a^i) \leq \log N^i_{T+1}(s^i, a^i)$. Note that for any $i \in [n]$, $N^i_{T+1}(s^i, a^i) \leq T$ and we have $2\mathsf{S}^i$ state-action pairs.
		Then, we have
		\begin{align*}
			\gamma & \leq 1 + \sum_{i \in [n]} \sum_{(s^i, a^i)} \gamma^i(s^i, a^i) \leq 1 + \sum_{i \in [n]} \sum_{(s^i, a^i)} \log N^i_{T+1}(s^i, a^i) \\
%			& \leq 1 + \sum_{i \in [n]} 2 \mathsf{S}^i \log \left( \sum_{(s^i, a^i)} \dfrac{N^i_{T+1}(s^i, a^i)}{2 \mathsf{S}^i} \right) \\
%			& \leq 1 + \sum_{i \in [n]} 2 \mathsf{S}^i \log \left( \sum_{(s^i, a^i)} N^i_{T+1}(s^i, a^i) \right) \\
			& = 1 + \sum_{i \in [n]} 2 \mathsf{S}^i \log T \leq 2 \bar {\mathsf{S}}_n \log T.
		\end{align*}
	\end{proof}
	
	\subsection{Bound on $\REGRET_0(T)$ (Lemma~\ref{lemma:R0_bound}.1)}
	
	We first state a basic property of Thompson sampling algorithms.
	\begin{lemma}[Thompson sampling Lemma~\cite{russo2014learning}]\label{lem:TS}
		Suppose the true parameters $\theta$ and the estimated ones $\theta_k$ have the same distribution given the same history~$\mathcal{H}$. For any $\mathcal{H}$-measurable function~$f$, we have 
		\[ \EXP[f(\theta)|\mathcal{H}] = \EXP[f(\theta_k)|\mathcal{H}]. \]
	\end{lemma}
	
	Now we consider $\REGRET_0(T)$. Let $\tilde{J}_\star = \alpha R_{\max} - J_\star$ and $\tilde{J}_{k} = \alpha R_{\max} - J_k$. By Lemma~\ref{lem:bound}, we have that $\tilde J_\star, \tilde J_k \in [0, R_{\max}]$. Therefore, 
	\begin{align*}
		\REGRET_{0}(T) & = \EXP\biggl[ T J_\star - \sum_{k=1}^{K_T} T_k J_k \biggr]
		= \EXP\biggl[ \sum_{k=1}^{K_T} T_k \tilde J_k - T \tilde J_\star \biggr]
		\notag = \sum_{k=1}^\infty \EXP\biggl[ \IND(\{t_k \le T \}) T_k \tilde J_k \biggr] - T \EXP\bigl[\tilde J_\star \bigr]
		\notag \\
		& \stackrel{(a)}\le \sum_{k=1}^\infty \EXP\biggl[ \IND(\{t_k \le T \}) (T_{k-1} + 1) \tilde J_k \biggr] - T \EXP\bigl[\tilde J_\star \bigr] \mathsf{D}_{\max}\stackrel{(b)}\le \sum_{k=1}^\infty \EXP\biggl[ \IND(\{t_k \le T \}) (T_{k-1} + 1) \tilde J_\star \biggr] - T \EXP\bigl[\tilde J_\star \bigr]
		\notag \\
		& = \EXP\biggl[ \sum_{k = 1}^{K_T} (T_{k-1} + 1) \tilde{J}_\star \biggr] - T \EXP\bigl[ \tilde{J}_\star \bigr] \stackrel{(c)}\le \EXP[ K_T \tilde J_\star ] \stackrel{(d)}\le \alpha R_{\max} \EXP[K_T],
	\end{align*}
	where $(a)$ uses the fact that due to the first stopping criterion, $T_k \le T_{k-1} + 1$; $(b)$ uses the Thompson sampling Lemma (Lemma~\ref{lem:TS}); $(c)$ uses the fact that $\sum_{k = 1}^{K_T} T_{k-1} \leq T$; and $(d)$ uses Lemma~\ref{lem:bound}. The result then follows by substituting the value of $K_T$ from Lemma~\ref{lemma:KT_bound_general}.
	
	\subsection{Bound on $\REGRET_1(T)$ (Lemma~\ref{lemma:R0_bound}.2)}
	$\REGRET_1(T)$ is a telescoping sum, which can be simplified as follows:
	\[
	\REGRET_1(T) = \EXP\bigg[ \sum_{k = 1}^{K_T} \sum_{t = t_k}^{t_{k+1}-1} \big[ {\bs V}_k(\bs S_{t+1}) - {\bs V}_k(\bs S_{t}) \big] \bigg] 
	= \EXP\bigg[ \sum_{k = 1}^{K_T} \big[ {\bs V}_k(\bs S_{t_{k+1}}) - {\bs V}_k(\bs S_{t_{k}}) \big] \bigg] \leq 2 \frac{\alpha R_{\max}}{1 - \lambda^*} \EXP[ K_T ]
	\]
	where the last inequality uses Lemma~\ref{lem:bound}. The result then follows by substituting the value of $K_T$ from Lemma~\ref{lemma:KT_bound_general}.
	
	\subsection{Bound on $\REGRET_2(T)$ (Lemma~\ref{lemma:R0_bound}.3)} \label{app:r2-proof}
	
	\subsubsection{Notation.} \label{subsubsec:note}
	For any arm~$i \in [n]$, let $N^i_t(s^i, a^i, s^i_{+})$ denote the number of times $(S_\tau, A_\tau, S_{\tau + 1})$ is equal to $(s^i, a^i, s^i_{+})$ until time~$t$. Let $\hat P^i_t(s^i_{+} | s^i, a^i) = N^i_t(s^i, a^i, s^i_{+})/\bigl( 1 \vee N^i_t(s^i, a^i) \bigr)$ denote the empirical distribution based on observations up to time~$t$. For the ease of notation, for a given $\delta \in (0,1)$, we define
	\begin{equation}\label{eq:eps}
		\epsilon^i_\delta(\ell) = 
        \sqrt{\frac{2 \mathsf{S}^i \log(1/\delta)}{1 \vee \ell}}.
        % \quad\text{and}\quad
        % \epsilon^i_k(\delta; s^i,a^i) =  \epsilon^i(\delta;
        % s^i,a^i,N^i_{t_k}(s^i, a^i)).  
	\end{equation}
	
	\subsubsection{Some preliminary results.}
	We start by basic property of difference of two expected values.
	
	\begin{lemma}\label{lem:TV}
		Let $p, q \in \Delta(\bs{\mathcal{S}})$. Then, for any function $f \colon \bs{\mathcal{S}} \to \reals$, we have
		\[
		\bigl| \bigl\langle f, p \bigr\rangle
		-
		\bigl\langle f, q \bigr\rangle \bigr|
		\le \frac12 \SPAN(f) 
		\bigl\lVert p - q \bigr\rVert_1.
		\]
	\end{lemma}
	\begin{proof}
		Let $\bar f = ( \max f + \min f)/2$. Then
		\begin{align*}
			\hskip 2em \hskip -2em
			\bigl| \bigl\langle f, p \bigr\rangle
			-
			\bigl\langle f, q \bigr\rangle \bigr|
			& =
			\bigl| \bigl\langle f - \bar f, p \bigr\rangle
			-
			\bigl\langle f - \bar f, q \bigr\rangle \bigr|
			= 
			\bigl| \bigl\langle f - \bar f, p - q \bigr\rangle \bigr|
			\le 
			\| f - \bar f \|_{\infty} \bigl| \bigl\langle \bs{1}, p - q \bigr\rangle \bigr|
			\\
			&\le
			\tfrac12 \SPAN(f) 
			\bigl| \bigl\langle \bs{1}, p - q \bigr\rangle \bigr|
			=
			\tfrac12 \SPAN(f) 
			\bigl\lVert p - q \bigr\rVert_1.
		\end{align*}
	\end{proof}
	
	\begin{lemma}\label{lem:ineq}
		Consider any arm~$i$, episode~$k$, $\delta \in (0,1)$, $\ell > 1$ and state-action pair $(s^i,a^i)$. 
		Define events $\mathcal{E}^i_{\ell} = \{ N^i_{t_k}(s^i, a^i) = \ell \}$ and $\mathcal{F}^i = \{ \lVert P^i(\cdot \, | s^i, a^i) - \hat P^i_{t_k}(\cdot \, | s^i, a^i) \rVert_1 \leq \epsilon_\delta(N^i_{t_k}(s^i, a^i)) \}$, and $\mathcal{F}^i_k = \{ \lVert P^i_k(\cdot \, | s^i, a^i) - \hat P^i_{t_k}(\cdot \, | s^i, a^i) \rVert_1 \leq \epsilon_\delta(N^i_{t_k}(s^i, a^i)) \}$. Then, we have
		\begin{align}
			\PR\Bigl( \bigl\lVert P^i(\cdot \, | s^i, a^i) - \hat P^i_{t_k}(\cdot \, | s^i, a^i) \bigr\rVert_1
			&>
            \epsilon^i_\delta(\ell) \Bigm| \mathcal{E}^i_{\ell}
			\Bigr) \le \delta,
			\label{eq:ineq-1} \\
			\PR\Bigl( \bigl\lVert P^i_k(\cdot \, | s^i, a^i) - \hat P^i_{t_k}(\cdot \, | s^i, a^i) \bigr\rVert_1
			&>
            \epsilon^i_\delta(\ell) \Bigm| \mathcal{E}^i_{\ell}
			\Bigr) \le \delta, \label{eq:ineq-2}
		\end{align}
        where $\epsilon^i_{\delta}(\ell)$ is given by~\eqref{eq:eps}. The above inequalities imply that
		\begin{align}
			\EXP\Bigl[ \bigl\lVert P^i(\cdot \, | s^i, a^i) - \hat P^i_{t_k}(\cdot \, | s^i, a^i) \bigr\rVert_1 \Bigm| \mathcal{F}^i \Bigr]  
			&\le
            \EXP[\epsilon^i_\delta(N^i_{t_k}(s^i, a^i)) | \mathcal{F}^i] + 2\delta,
			\label{eq:ineq-3} \\
			\EXP\Bigl[ \bigl\lVert P^i_k(\cdot \, | s^i, a^i) - \hat P^i_{t_k}(\cdot \, | s^i, a^i) \bigr\rVert_1 \Bigm| \mathcal{F}^i_k \Bigr]
			&\le
            \EXP[\epsilon^i_\delta(N^i_{t_k}(s^i, a^i)) | \mathcal{F}^i_k] + 2\delta.
			\label{eq:ineq-4}
		\end{align}
	\end{lemma}
	\begin{proof}
      Given arm~$i$, state~$s^i$ of the arm and action~$a^i$ chosen for the arm, we know from \cite{weissman2003inequalities} that for any $\varepsilon > 0$, the $L1$-deviation of the true distribution and the empirical distribution over $\mathcal{S}^i$ states from $N^{i}_{t_k}(s^i, a^i) = \ell$ samples is bounded by 
        \[ 
        \mathbb{P}\Big( \lVert P^i(\cdot \, | s^i, a^i) - \hat P^i_{t_k}(\cdot \, | s^i, a^i) \rVert_1 \geq \varepsilon \Bigm| \mathcal{E}^i_{\ell} \Big) 
		\leq 2^{\mathsf{S}^i} \exp\biggl( - \dfrac{\ell \varepsilon^2}{2} \biggr) 
		< \exp\biggl( \mathsf{S}^i - \dfrac{\ell \varepsilon^2}{2} \biggr).
		\]
        Therefore, setting $\delta = \exp(S^i - \ell \varepsilon^2/2)$, we get
		\[
		\PR\Biggl( \bigl\lVert P^i(\cdot \, | s^i, a^i) - \hat P^i_{t_k}(\cdot \, | s^i, a^i) \bigr\rVert_1
		>
		\sqrt{\frac{2 (\mathsf{S}^i + \log(1/\delta))}{1 \vee \ell}} 
		\Bigm| \mathcal{E}^i_{\ell} \Biggr) \le \delta.
		\]
		Note that, $\mathsf{S}^i \ge 2$, therefore $\mathsf{S^i} + \log(1/\delta) \le \mathsf{S^i} \log(1/\delta)$. Hence,
		\begin{multline*}
			\PR\Biggl( \bigl\lVert P^i(\cdot \, | s^i, a^i) - \hat P^i_{t_k}(\cdot \, | s^i, a^i) \bigr\rVert_1
			>
			\sqrt{\frac{2 \mathsf{S}^i \log(1/\delta)}{1 \vee \ell}} 
			\Biggm| \mathcal{E}^i_{\ell} \Biggr) 
			\\ 
			<
			\PR\Biggl( \bigl\lVert P^i(\cdot \, | s^i, a^i) - \hat P^i_{t_k}(\cdot \, | s^i, a^i) \bigr\rVert_1 > 
			\sqrt{\frac{2 (\mathsf{S}^i + \log(1/\delta))}{1 \vee \ell}} 
			\Biggm| \mathcal{E}^i_{\ell} \Biggr) 
			\le \delta.
		\end{multline*}
		This proves~\eqref{eq:ineq-1}. Eq.~\eqref{eq:ineq-2} follows from the Thompson sampling Lemma (Lemma~\ref{lem:TS}). 
		
		To prove \eqref{eq:ineq-3} and \eqref{eq:ineq-4}, we first show
		\begin{align} \label{eqn:PFc_bound}
			P((\mathcal{F}^i)^c) & = \PR\Bigl( \bigl\lVert P^i(\cdot \, | s^i, a^i) - \hat P^i_{t_k}(\cdot \, | s^i, a^i) \bigr\rVert_1 > \epsilon^i_\delta(N^i_{t_k}(s^i, a^i)) \Bigr) \notag \\
			& = \sum_{\ell = 1}^{\infty} \PR\Bigl( \bigl\lVert P^i(\cdot \, | s^i, a^i) - \hat P^i_{t_k}(\cdot \, | s^i, a^i) \bigr\rVert_1 > \epsilon^i_\delta(N^i_{t_k}(s^i, a^i)) \Bigm| \mathcal{E}^i_{\ell} \Bigr) \PR(\mathcal{E}^i_{\ell}) \notag \notag \\
			& = \sum_{\ell = 1}^{\infty} \PR\Bigl( \bigl\lVert P^i(\cdot \, | s^i, a^i) - \hat P^i_{t_k}(\cdot \, | s^i, a^i) \bigr\rVert_1 > \epsilon^i_\delta(\ell) \Bigm| \mathcal{E}^i_{\ell} \Bigr) \PR(\mathcal{E}^i_{\ell}) \notag \\
			& \leq \sum_{\ell = 1}^{\infty} \delta \PR(\mathcal{E}^i_{\ell}) = \delta.
		\end{align}
		Now consider
		\begin{align*}
            \hskip 2em & \hskip -2em
			\EXP\Bigl[ \bigl\lVert P^i(\cdot \, | s^i, a^i) - \hat P^i_{t_k}(\cdot \, | s^i, a^i) \bigr\rVert_1 \Bigm| \mathcal{F}^i \Bigr]
			=
			\EXP\Bigl[ \bigl\lVert P^i(\cdot \, | s^i, a^i) - \hat P^i_{t_k}(\cdot \, | s^i, a^i) \bigr\rVert_1 \Bigm| \mathcal{F}^i \Bigr] \PR(\mathcal{F}^i) \notag \\
			& + \EXP\Bigl[ \bigl\lVert P^i(\cdot \, | s^i, a^i) - \hat P^i_{t_k}(\cdot \, | s^i, a^i) \bigr\rVert_1 \Bigm| (\mathcal{F}^i)^c \Bigr] \PR((\mathcal{F}^i)^c) \notag \\
			& \stackrel{(a)}\le 2 P((\mathcal{F})^c) + \EXP[\epsilon^i_\delta(N^i_{t_k}(s^i, a^i)) | \mathcal{F}^i] \notag \\
            & \stackrel{(b)}\le 2 \delta + \EXP[\epsilon^i_\delta(N^i_{t_k}(s^i, a^i)) | \mathcal{F}^i],
		\end{align*}
		where $(a)$ uses $\lVert \cdot \rVert_1 \le 2$ and $P(\mathcal{F}) \le 1$ and $(b)$ uses \eqref{eqn:PFc_bound}. This proves \eqref{eq:ineq-3}. Eq.~\eqref{eq:ineq-4} follows from a similar argument.
		%
		%		As a result, we define a confidence set for bandit~$i$ in episode~$k$ as
		%		\[ \Theta^i_k = \bigg\{ P^i \bigg| \forall s^i \in \mathcal{S}^i, \forall a^i \in \{0, 1\}, \lVert P^i(\cdot \, | s^i, a^i) - \hat P^i_{t_k}(\cdot \, | s^i, a^i) \rVert_1 \leq \epsilon^i_k(s^i, a^i) \bigg\} \]
		%		where $\epsilon^i_k(s^i, a^i) = \sqrt{\dfrac{2 \mathsf{S}^i \log (1/\delta)}{1 \vee N^i_{t_k}(s^i, a^i)}}$. 
		%		Then, we have
		%		\[ \mathbb{P}\bigg( \bigl\lVert {P}^{i}(\cdot \, | s^i, a^i) - \hat P^i_{t_k}(\cdot \, | s^i, a^i) \bigr\rVert_1 \geq \epsilon^i_k(s^i, a^i) \bigg) \leq \delta. \]
		%		Similarly, we have
		%		\[ \mathbb{P}\bigg( \bigl\lVert {P}^{i}_k(\cdot \, | s^i, a^i) - \hat P^i_{t_k}(\cdot \, | s^i, a^i) \bigr\rVert_1 \geq \epsilon^i_k(s^i, a^i) \bigg) \leq \delta. \]
	\end{proof}
	
	\begin{lemma} \label{lemma:Pdiff}
		Consider episode~$k$, $\delta \in (0,1)$, and joint state-action pair~$(\bs{s}, \bs{a})$. Define events $\mathcal{F}^i$ and $\mathcal{F}^i_k$ as in Lemma~\ref{lem:ineq}. Then we have 
		\begin{align*}
			\EXP\Bigl[ \bigl\lVert
			\bs{P}_\star( \cdot | \bs{s}, \bs{a})
			-
			\bs{P}_k(\cdot | \bs{s}, \bs{a}) 
			\bigr\rVert_1 \Bigr] 
			& = \sum_{i \in [n]} \EXP\Bigl[ \bigl\lVert {P}^{i}(\cdot \, | s^i, a^i) - P^i_{k}(\cdot \, | s^i, a^i) \bigr\rVert_1 \Bigr] \\ 
			& \le 4 n \delta + \sum_{i \in [n]} \left( \EXP\bigl[ \epsilon^i_\delta(N^i_{t_k}(s^i, a^i)) | \mathcal{F}^i \bigr] + \EXP\bigl[ \epsilon^i_\delta(N^i_{t_k}(s^i, a^i)) | \mathcal{F}^i_k \bigr] \right). 
		\end{align*}
	\end{lemma}
	\begin{proof}
		The first equality follows from \cite[Lemma 13]{jung2019thompson}. Then, the rest of the proof is as follows:
		\begin{align*}
			\EXP\Bigl[ \bigl\lVert
			\bs{P}_\star( \cdot | \bs{s}, \bs{a})
			-
			\bs{P}_k(\cdot | \bs{s}, \bs{a}) 
			\bigr\rVert_1 \Bigr] 
			& = \sum_{i \in [n]} \EXP\Bigl[ \bigl\lVert {P}^{i}(\cdot \, | s^i, a^i) - P^i_{k}(\cdot \, | s^i, a^i) \bigr\rVert_1 \Bigr] \notag \\
			& \leq \sum_{i \in [n]} \EXP\Bigl[ \bigl\lVert {P}^{i}(\cdot \, | s^i, a^i) - \hat P^i_{t_k}(\cdot \, | s^i, a^i) \bigr\rVert_1 + \bigl\lVert P^{i}_k(\cdot \, | s^i, a^i) - \hat P^i_{t_k}(\cdot \, | s^i, a^i) \bigr\rVert_1 \Bigr] \\
			&\le 4 n \delta + \sum_{i \in [n]} \left( \EXP\bigl[ \epsilon^i_\delta(N^i_{t_k}(s^i, a^i)) | \mathcal{F}^i \bigr] + \EXP\bigl[ \epsilon^i_\delta(N^i_{t_k}(s^i, a^i)) | \mathcal{F}^i_k \bigr] \right),
		\end{align*}
		where the first inequality follows from triangle inequality, and the second follows from Lemma~\ref{lem:ineq}.
	\end{proof}
	
	\subsubsection{Bounding $\REGRET_2(T)$.} \label{subsec:R2}
	Now, consider the inner summation in the expression for $\REGRET_2(T)$:
	\begin{align}
		\hskip 2em & \hskip -2em
		\EXP\Bigl[ 
		\bigl\langle 
		\bs{P}_k(\cdot \, | \bs{s}, \bs{a}), \bs{V}_k \bigr\rangle
		-
		\bs{V}_k(\bs{S}_{t+1})
		\Bigr]
		= 
		\EXP\Bigl[ 
		\bigl\langle 
		\bs{P}_k(\cdot \, | \bs{s}, \bs{a}), \bs{V}_k \bigr\rangle
		-
		\bigl\langle 
		\bs{P}_\star(\cdot \, | \bs{s}, \bs{a}), \bs{V}_k \bigr\rangle
		\Bigr]
		\notag \\
		&\stackrel{(a)}\le \EXP\Bigl[ \tfrac12 \SPAN(\bs{V}_k) 
		\bigl\lVert
		\bs{P}_k(\cdot | \bs{s}, \bs{a}) 
		-
		\bs{P}_\star( \cdot | \bs{s}, \bs{a})
		\bigr\rVert_1 \Bigr]
		\notag \\
		&\stackrel{(b)} \le \frac{\alpha R_{\max}}{1 - \lambda^*}
		\EXP\Bigl[ \bigl\lVert
		\bs{P}_k(\cdot | \bs{s}, \bs{a}) 
		-
		\bs{P}_\star( \cdot | \bs{s}, \bs{a})
		\bigr\rVert_1 \Bigr]\label{eq:R2:1} 
	\end{align}
	where $(a)$ follows from Lemma~\ref{lem:TV} and $(b)$ follows from Lemma~\ref{lem:bound}.
	Then, by Lemma~\ref{lemma:Pdiff}, we have
	\begin{align} \label{eqn:R2-bound}
		\REGRET_2(T) & = \EXP\biggl[\sum_{k=1}^{K_T} \sum_{t=t_k}^{t_{k+1} - 1} 
		\bigl[ \bs{P}_k \bs{V}_k \bigr](\bs{S}_{t}) - \bs{V}_k(\bs{S}_{t+1}) \biggr] \notag \\
		& \leq \frac{\alpha R_{\max}}{1 - \lambda^*} \sum_{k=1}^{K_T} \sum_{t=t_k}^{t_{k+1} - 1} \Big( 4 n \delta
		+ \sum_{i \in [n]} \left( \EXP\bigl[ \epsilon^i_\delta(N^i_{t_k}(s^i, a^i)) | \mathcal{F}^i \bigr] + \EXP\bigl[ \epsilon^i_\delta(N^i_{t_k}(s^i, a^i)) | \mathcal{F}^i_k \bigr] \right) \Big).
	\end{align}
	
	For the first inner term of \eqref{eqn:R2-bound}, we have
	\begin{equation} \label{eqn:R2-2}
		\sum_{k = 1}^{K_T} \sum_{t = t_k}^{t_{k+1}-1} 4 n \delta = \sum_{t = 1}^{T} 4 n \delta = 4 n \delta T.
	\end{equation}
	
	For the second inner term of \eqref{eqn:R2-bound}, we have 
	\begin{align} \label{eqn:R2-1}
		\sum_{k = 1}^{K_T} \sum_{t = t_k}^{t_{k+1}-1} & \sum_{i \in [n]} \EXP\bigg[ \sqrt{\dfrac{2 \mathsf{S}^i \log (1/\delta)}{1 \vee N^i_{t_k}(S^i_t, A^i_t)}} \Biggm| \mathcal{F}^i \bigg] \leq \sum_{k = 1}^{K_T} \sum_{t = t_k}^{t_{k+1}-1} \sum_{i \in [n]} \EXP\bigg[ \sqrt{\dfrac{4 \mathsf{S}^i \log (1/\delta)}{1 \vee N^i_{t}(S^i_t, A^i_t)}} \Biggm| \mathcal{F}^i \bigg] \notag \\
		& = \sum_{t = 1}^{T} \sum_{i \in [n]} \EXP\bigg[ \sqrt{\dfrac{4 \mathsf{S}^i \log (1/\delta)}{1 \vee N^i_{t}(S^i_t, A^i_t)}} \Biggm| \mathcal{F}^i \bigg] \notag \\
		& \leq \sqrt{4 \log (1/\delta)} \sum_{i \in [n]} \sqrt{\mathsf{S}^i} \sum_{t = 1}^{T} \EXP\bigg[ \sqrt{\dfrac{1}{1 \vee N^i_{t}(S^i_t, A^i_t)}} \Biggm| \mathcal{F}^i \bigg] \notag \\
		& = \sqrt{4 \log (1/\delta)} \sum_{i \in [n]} \sqrt{\mathsf{S}^i} \sum_{(s^i, a^i)} \sum_{t = 1}^{T} \EXP\bigg[ \IND(S^i_t = s^i, A^i_t = a^i) \sqrt{\dfrac{1}{1 \vee N^i_{t}(s^i, a^i)}} \Biggm| \mathcal{F}^i \bigg] \notag \\
		& = \sqrt{4 \log (1/\delta)} \sum_{i \in [n]} \sqrt{\mathsf{S}^i} \sum_{(s^i, a^i)} \EXP\bigg[ \IND(N^i_{T+1}(s^i, a^i) > 0) + \sum_{j = 1}^{N^i_{T+1}(s^i, a^i)-1} \dfrac{1}{\sqrt{j}} \Biggm| \mathcal{F}^i \bigg] \notag \\
		& \leq \sqrt{4 \log (1/\delta)} \sum_{i \in [n]} \sqrt{\mathsf{S}^i} \sum_{(s^i, a^i)} \EXP\bigg[ \IND(N^i_{T+1}(s^i, a^i) > 0) + 2 \sqrt{N^i_{T+1}(s^i, a^i)} \Biggm| \mathcal{F}^i \bigg] \notag \\
		& \leq \sqrt{4 \log (1/\delta)} \sum_{i \in [n]} \sqrt{\mathsf{S}^i} \sum_{(s^i, a^i)} 3 \EXP\bigg[ \sqrt{N^i_{T+1}(s^i, a^i)} \Biggm| \mathcal{F}^i \bigg] \notag \\
		& \stackrel{(a)}\leq 6 \sqrt{\log (1/\delta)} \sum_{i \in [n]} \sqrt{\mathsf{S}^i} \EXP\bigg[ \sqrt{2 \mathsf{S}^i \sum_{(s^i, a^i)} N^i_{T+1}(s^i, a^i)} \Biggm| \mathcal{F}^i \bigg] \stackrel{(b)}= 6\sqrt{2} \bar {\mathsf{S}}_n \sqrt{T \log (1/\delta)}
	\end{align}
	where $(a)$ uses Cauchy-Schwartz inequality and $(b)$ uses the fact that $\sum_{(s^i, a^i)} N^i_{T+1}(s^i, a^i) = T$. The same approach works for the third inner term of \eqref{eqn:R2-bound} and we get
	\begin{align} \label{eqn:R2-11}
		\sum_{k = 1}^{K_T} \sum_{t = t_k}^{t_{k+1}-1} \sum_{i \in [n]} \EXP\bigg[ \sqrt{\dfrac{2 \mathsf{S}^i \log (1/\delta)}{1 \vee N^i_{t_k}(S^i_t, A^i_t)}} \Biggm| (\mathcal{F}^i)^c \bigg] \leq 6\sqrt{2} \bar {\mathsf{S}}_n \sqrt{T \log (1/\delta)}.
	\end{align}
	
	Finally, by setting $\delta = 1/T$, and substituting \eqref{eqn:R2-11}, \eqref{eqn:R2-1} and \eqref{eqn:R2-2} in \eqref{eqn:R2-bound}, we get
	\begin{align*}
		\REGRET_2(T) & \leq \frac{\alpha R_{\max}}{1 - \lambda^*} \bigl( 4n + 12\sqrt{2} \bar {\mathsf{S}}_n \sqrt{T \log T} \bigr) \leq 12\sqrt{2} \frac{\alpha R_{\max}}{1 - \lambda^*} \left( n + \bar {\mathsf{S}}_n \sqrt{T \log T} \right).
	\end{align*}
	%	\textcolor{red}{One approach is to consider $\delta = 1/(T^{1.5} n^{\alpha})$ and optimize $\alpha$ to remove $n^2$ in the second term which is not possible. The other approach is to consider $\delta = (1/T)^{f(n)}$ and optimize $f(n)$ to remove $n^2$ in the second term. Firstly, $f(n) \geq 1.5$ to limit the growth of the first term. Then, we should have $\mathcal{O}(n^2 \sqrt{f(n)}) = \mathcal{O}(n \sqrt{n})$ for which we need $f(n) = 1/n$ which contradicts $f(n) \geq 1.5$. As a result, it seems we cannot optimize $\delta$ to prevent growth of $n$.}
	
	\subsection{Bound on $\REGRET_2(T)$ (Lemma~\ref{lemma:R2_bound_tight})} \label{app:r2-proof2}
	
	\subsubsection{Notation.}
	For any arm~$i \in [n]$, let $N^i_t(s^i, a^i, s^i_{+})$ and $\hat P^i_t(s^i_{+} | s^i, a^i)$ denote the same variables as defined in Section~\ref{subsubsec:note}.
	
	\subsubsection{Some preliminary results.}
	
	\begin{lemma} \label{lem:LipK}
		For any Lipschitz function~$f: \mathcal{X} \to \mathbb{R}$ with Lipschitz coefficient $L_f$, and any probability measures $\gamma$ and $\zeta$ on $(\mathcal{X}, d_X)$ we have
		\begin{equation*}
			\bigg| \sum_{x \in \mathcal{X}} f(x) \gamma(x) - \sum_{x \in \mathcal{X}} f(x) \zeta(x) \bigg| \leq L_f \mathcal{K}(\gamma, \zeta).
		\end{equation*}
	\end{lemma}
	The result is immediately derived from the definition of Kantorovich distance.
	
	\begin{lemma} \label{lem:kant-bound}
		Let $\nu$ denote a probability measure on $(\mathbb{R}, |\cdot|)$ and let $\hat{\nu}_{n}$ denote the estimated probability measure by $n$ samples from $\nu$. Then, for all $n \geq 1$ and all $\epsilon > 0$, there exist constants $C$ and $c$ which depend on $\nu$ such that
		\[ \mathbb{P}\left( \mathcal{K}(\nu, \hat{\nu}_{n}) \geq \epsilon \right) \leq C \exp(-cn\epsilon) \IND(\epsilon \leq 1) + C \exp(-cn\epsilon^2) \IND(\epsilon > 1). \]
	\end{lemma}
	\begin{proof}
		The lemma follows directly by applying \cite[Theorem 2]{fournier2015rate} and setting $d = 1$, $p = 1$ and $\alpha = 2$ which satisfies condition (C1) of the cited paper for our model.
	\end{proof}
	
	Let 
	\begin{equation} \label{eqn:eps2}
		\epsilon^i_\delta(\ell) = \sqrt{\dfrac{ \log (C /\delta)}{c(1 \vee \ell)}}.
	\end{equation}
	\begin{lemma} \label{lem:ineq2}
		Consider any arm~$i$, episode~$k$, $\delta \in (0,1)$, $\ell > 1$ and state-action pair $(s^i,a^i)$. 
		Define events $\mathcal{E}^i_{\ell}$ and $\mathcal{F}^i$, and $\mathcal{F}^i_k$ as in Lemma~\ref{lem:ineq}. Then, we have
		\begin{align}
			\PR\Bigl( \mathcal{K}( P^i(\cdot \, | s^i, a^i), \hat P^i_{t_k}(\cdot \, | s^i, a^i) )
			&>
			\epsilon^i_\delta(\ell) \Bigm| \mathcal{E}^i_{\ell}
			\Bigr) \le \delta,
			\label{eq:ineq-11} \\
			\PR\Bigl( \mathcal{K}( P^i_k(\cdot \, | s^i, a^i), \hat P^i_{t_k}(\cdot \, | s^i, a^i) )
			&>
			\epsilon^i_\delta(\ell) \Bigm| \mathcal{E}^i_{\ell}
			\Bigr) \le \delta, \label{eq:ineq-21}
		\end{align}
		where $\epsilon^i_{\delta}(\ell)$ is given by~\eqref{eqn:eps2}. Furthermore, 
		the above inequalities imply that
		\begin{align}
			\EXP\Bigl[ \mathcal{K}( P^i(\cdot \, | s^i, a^i), \hat P^i_{t_k}(\cdot \, | s^i, a^i) ) \Bigm| \mathcal{F}^i \Bigr]  
			&\le
			\EXP[\epsilon^i_\delta(\ell) | \mathcal{F}^i] + 2\DIAM(\mathcal{S}^i)\delta,
			\label{eq:ineq-31} \\
			\EXP\Bigl[ \mathcal{K}( P^i_k(\cdot \, | s^i, a^i), \hat P^i_{t_k}(\cdot \, | s^i, a^i) ) \Bigm| \mathcal{F}^i_k \Bigr]
			&\le
			\EXP[\epsilon^i_\delta(\ell) | \mathcal{F}^i_k] + 2\DIAM(\mathcal{S}^i)\delta.
			\label{eq:ineq-41}
		\end{align}
	\end{lemma}
	\begin{proof}
		Given arm~$i$, state~$s^i$ of the arm and action~$a^i$ chosen for the arm, we know from Lemma~\ref{lem:kant-bound} that for any $\epsilon > 0$, the Kantorovich distance between the true distribution and the empirical distribution over $\mathcal{S}^i$ with $N^i_{t_k}(s^i, a^i) = \ell$ samples is bounded by 
		\begin{align*}
			\mathbb{P}\Big( \mathcal{K}(P^i(\cdot \, | s^i, a^i), \hat P^i_{t_k}(\cdot \, | s^i, a^i)) > \epsilon \Bigm| \mathcal{E}^i_{\ell} \Big) 
			& < C \exp(-c \ell \epsilon) \IND(\epsilon \leq 1) + C \exp(-c \ell \epsilon^2) \IND(\epsilon > 1).
		\end{align*}
		which can also be bounded by
		\[ 
		\mathbb{P}\Big( \mathcal{K}(P^i(\cdot \, | s^i, a^i), \hat P^i_{t_k}(\cdot \, | s^i, a^i)) > \epsilon^i_\delta(N^i_{t_k}(s^i, a^i)) \Bigm| \mathcal{E}^i_{\ell} \Big) 
		< C \exp(-c \ell \epsilon^2).
		\]
		Therefore, setting $\delta = C \exp(-c \ell \epsilon^2)$
		\[
		\PR\Biggl( \mathcal{K}(P^i(\cdot \, | s^i, a^i), \hat P^i_{t_k}(\cdot \, | s^i, a^i)) 
		> \sqrt{\dfrac{ \log (C /\delta)}{c(1 \vee \ell)}} \Biggm| \mathcal{E}^i_{\ell}
		\Biggr) < \delta,
		\]
		
		This proves~\eqref{eq:ineq-11}. Eq.~\eqref{eq:ineq-21} follows from the Thompson sampling Lemma (Lemma~\ref{lem:TS}). 
		
		To prove \eqref{eq:ineq-31} and \eqref{eq:ineq-41}, we first show
		\begin{align} \label{eqn:PFc_bound2}
			P((\mathcal{F}^i)^c) & = \PR\Bigl( \mathcal{K}(P^i(\cdot \, | s^i, a^i), \hat P^i_{t_k}(\cdot \, | s^i, a^i)) > \epsilon^i_\delta(N^i_{t_k}(s^i, a^i)) \Bigr) \notag \\
			& = \sum_{\ell = 1}^{\infty} \PR\Bigl( \mathcal{K}(P^i(\cdot \, | s^i, a^i), \hat P^i_{t_k}(\cdot \, | s^i, a^i)) > \epsilon^i_\delta(N^i_{t_k}(s^i, a^i)) \Bigm| \mathcal{E}^i_{\ell} \Bigr) \PR(\mathcal{E}^i_{\ell}) \notag \\
			& = \sum_{\ell = 1}^{\infty} \PR\Bigl( \mathcal{K}(P^i(\cdot \, | s^i, a^i), \hat P^i_{t_k}(\cdot \, | s^i, a^i)) > \epsilon^i_\delta(\ell) \Bigm| \mathcal{E}^i_{\ell} \Bigr) \PR(\mathcal{E}^i_{\ell}) \notag \\
			& \leq \sum_{\ell = 1}^{\infty} \delta \PR(\mathcal{E}^i_{\ell}) = \delta.
		\end{align}
	
		Also, note that the for any two distributions $\nu_1$ and $\nu_2$ defined on $\mathcal{S}^i$, 
		\begin{equation} \label{eqn:K_ineq}
			\mathcal{K}(\nu_1,\nu_2) \le \DIAM(\mathcal{S}^i) \bigl\| \nu_1 - \nu_2 \bigr\|_1
			\le 2 \DIAM(\mathcal{S}^i).
		\end{equation}
	
		Now consider
		\begin{align*}
			\hskip 2em & \hskip -2em
			\EXP\Bigl[ \mathcal{K}(P^i(\cdot \, | s^i, a^i), \hat P^i_{t_k}(\cdot \, | s^i, a^i)) \Bigm| \mathcal{F}^i \Bigr]
			=
			\EXP\Bigl[ \mathcal{K}(P^i(\cdot \, | s^i, a^i), \hat P^i_{t_k}(\cdot \, | s^i, a^i)) \Bigm| \mathcal{F}^i \Bigr] \PR(\mathcal{F}^i) \notag \\
			& + \EXP\Bigl[ \mathcal{K}(P^i(\cdot \, | s^i, a^i), \hat P^i_{t_k}(\cdot \, | s^i, a^i)) \Bigm| (\mathcal{F}^i)^c \Bigr] \PR((\mathcal{F}^i)^c) \notag \\
			& \stackrel{(a)}\le 2 \DIAM(\mathcal{S}^i) P((\mathcal{F})^c) + \EXP[\epsilon^i_\delta(N^i_{t_k}(s^i, a^i)) | \mathcal{F}^i] \notag \\
			& \stackrel{(b)}\le 2 \DIAM(\mathcal{S}^i) \delta + \EXP[\epsilon^i_\delta(N^i_{t_k}(s^i, a^i)) | \mathcal{F}^i],
		\end{align*}
		where $(a)$ uses \eqref{eqn:K_ineq} and $P(\mathcal{F}) \le 1$ and $(b)$ uses \eqref{eqn:PFc_bound2}. This proves \eqref{eq:ineq-31}. Eq.~\eqref{eq:ineq-41} follows from a similar argument.
	\end{proof}
	
	\begin{lemma} \label{lemma:Pdiff2}
		For any episode~$k$, and $\delta \in (0,1)$, we have 
		\begin{align*}
			\EXP\Bigl[ \mathcal{K}( \bs{P}_\star( \cdot | \bs{s}, \bs{a}), \bs{P}_k(\cdot | \bs{s}, \bs{a}) ) \Bigr] 
			& = \sum_{i \in [n]} \EXP\Bigl[ \mathcal{K}( {P}^{i}(\cdot \, | s^i, a^i), P^i_{t_k}(\cdot \, | s^i, a^i) ) \Bigr] \\
			& \le 4 n \mathsf{D}_{\max} \delta + \sum_{i \in [n]} \left( \EXP\bigl[ \epsilon^i_\delta(N^i_{t_k}(s^i, a^i)) \Bigm| \mathcal{F}^i \bigr] + \EXP\bigl[ \epsilon^i_\delta(N^i_{t_k}(s^i, a^i)) \Bigm| \mathcal{F}^i_k \bigr] \right).
		\end{align*}
	\end{lemma}
	\begin{proof}
		The first equality follows from \cite[Lemma 4]{robustness2022sinha}. Then, the rest of the proof is similar to the proof of Lemma~\ref{lemma:Pdiff} by using Lemma~\ref{lem:ineq2} results and the following triangle inequality
		\begin{align*}
			\EXP\Bigl[ \mathcal{K}( {P}^{i}(\cdot \, | s^i, a^i), P^i_{t_k}(\cdot \, | s^i, a^i) ) \Bigr] \leq \EXP\Bigl[ \mathcal{K}( {P}^{i}(\cdot \, | s^i, a^i), \hat P^i_{t_k}(\cdot \, | s^i, a^i) ) + \mathcal{K}( P_k^{i}(\cdot \, | s^i, a^i), \hat P^i_{t_k}(\cdot \, | s^i, a^i) ) \Bigr].
		\end{align*}
	\end{proof}
	
	\subsubsection{Bounding $\REGRET_2(T)$.} \label{subsec:R22}
	First, consider the inner summation in the expression for $\REGRET_2(T)$:
	\begin{align}
		\hskip 2em & \hskip -2em
		\EXP\Bigl[
		\bigl\langle 
		\bs{P}_k(\cdot \, | \bs{S}_t, \bs{A}_t), \bs{V}_k \bigr\rangle
		-
		\bs{V}_k(\bs{S}_{t+1}) 
		\Bigr]
		= 
		\EXP\Bigl[ 
		\bigl\langle 
		\bs{P}_k(\cdot \, | \bs{S}_t, \bs{A}_t), \bs{V}_k \bigr\rangle
		-
		\bigl\langle 
		\bs{P}_\star(\cdot \, | \bs{S}_t, \bs{A}_t), \bs{V}_k \bigr\rangle
		\Bigr]
		\notag \\
		&\stackrel{(a)}\le \bs{L}_v \EXP\Bigl[
		\mathcal{K}\left(
		\bs{P}_k(\cdot | \bs{S}_t, \bs{A}_t) 
		-
		\bs{P}_\star( \cdot | \bs{S}_t, \bs{A}_t)
		\right) \Bigr] \label{eq:R2:12} 
	\end{align}
	where $(a)$ follows from Lemma~\ref{lem:LipK}.
	Then, by Lemma~\ref{lemma:Pdiff2}, we have
	\begin{align} \label{eqn:R2-bound2}
		\REGRET_2(T) & = \EXP\biggl[\sum_{k=1}^{K_T} \sum_{t=t_k}^{t_{k+1} - 1} 
		\bigl[ \bs{P}_k \bs{V}_k \bigr](\bs{S}_{t}) - \bs{V}_k(\bs{S}_{t+1}) \biggr] \notag \\
		& \leq \bs{L}_v \sum_{k=1}^{K_T} \sum_{t=t_k}^{t_{k+1} - 1} \left( 4 n \mathsf{D}_{\max} \delta + \sum_{i \in [n]} \sum_{i \in [n]} \left( \EXP\bigl[ \epsilon^i_\delta(N^i_{t_k}(s^i, a^i)) \Bigm| \mathcal{F}^i \bigr] + \EXP\bigl[ \epsilon^i_\delta(N^i_{t_k}(s^i, a^i)) \Bigm| \mathcal{F}^i_k \bigr] \right) \right).
	\end{align}
	
	For the first inner term of \eqref{eqn:R2-bound2}, we have
	\begin{equation} \label{eqn:R2-22}
		\sum_{k = 1}^{K_T} \sum_{t = t_k}^{t_{k+1}-1} 4 n \mathsf{D}_{\max} \delta = \sum_{t = 1}^{T} 4 n \mathsf{D}_{\max} \delta = 4 n \mathsf{D}_{\max} \delta T.
	\end{equation}
	
	Then, for the second inner term of \eqref{eqn:R2-bound2}, we have
	\begin{align} \label{eqn:R2-12}
		\sum_{k = 1}^{K_T} \sum_{t = t_k}^{t_{k+1}-1} & \sum_{i \in [n]} \EXP\bigg[ \sqrt{\dfrac{ \log (C /\delta)}{c(1 \vee N^i_{t_k}(S^i_t, A^i_t))}} \Biggm| \mathcal{F}^i \bigg] 
		\leq \sum_{k = 1}^{K_T} \sum_{t = t_k}^{t_{k+1}-1} \sum_{i \in [n]} \EXP\bigg[ \sqrt{\dfrac{ 2 \log (C /\delta)}{c(1 \vee N^i_{t}(S^i_t, A^i_t))}} \Biggm| \mathcal{F}^i \bigg] \notag \\
		& = \sum_{t = 1}^{T} \sum_{i \in [n]} \EXP\bigg[ \sqrt{\dfrac{ 2 \log (C /\delta)}{c(1 \vee N^i_{t}(S^i_t, A^i_t))}} \Biggm| \mathcal{F}^i \bigg] \notag \\
		& \leq \sum_{i \in [n]} \sqrt{\dfrac{2 \log (C /\delta)}{c}} \sum_{t = 1}^{T} \EXP\bigg[ \sqrt{\dfrac{1}{1 \vee N^i_{t}(S^i_t, A^i_t)}} \Biggm| \mathcal{F}^i \bigg] \notag \\
		& \stackrel{(a)}\leq \sum_{i \in [n]} \sqrt{\dfrac{2 \log (C /\delta)}{c}} \sum_{(s^i, a^i)} 3 \EXP\bigg[ \sqrt{N^i_{T+1}(s^i, a^i)} \Biggm| \mathcal{F}^i \bigg] \notag \\
		& \stackrel{(b)}\leq 3 \sqrt{2} \sum_{i \in [n]}  \sqrt{\dfrac{\log (C /\delta)}{c}} \EXP\bigg[ \sqrt{2 \mathsf{S}^i \sum_{(s^i, a^i)} N^i_{T+1}(s^i, a^i)} \Biggm| \mathcal{F}^i \bigg] \notag \\
		& \stackrel{(c)} = 6 \sum_{i \in [n]} \sqrt{\dfrac{\mathsf{S}^i \log (C /\delta) T}{c}} \stackrel{(d)}\leq 6 \sqrt{\dfrac{\bar{\mathsf{S}}_{n} \log (C /\delta) T}{c}}
	\end{align}
	where $(a)$ uses the result of Section~\ref{subsec:R2}, $(b)$ and $(d)$ use Cauchy-Schwartz inequality, and $(c)$ uses the fact that $\sum_{(s^i, a^i)} N^i_{T+1}(s^i, a^i) = T$.
	
	The same approach works for the third inner term of \eqref{eqn:R2-bound2} and we get
	\begin{align} \label{eqn:R2-13}
		\sum_{k = 1}^{K_T} \sum_{t = t_k}^{t_{k+1}-1} & \sum_{i \in [n]} \EXP\bigg[ \sqrt{\dfrac{ \log (C /\delta)}{c(1 \vee N^i_{t_k}(S^i_t, A^i_t))}} \Biggm| \mathcal{F}^i \bigg] \leq 6 \sqrt{\dfrac{\bar{\mathsf{S}}_n \log (C /\delta) T}{c}}.
	\end{align}
	
	Finally, by setting $\delta = 1/T$, and substituting \eqref{eqn:R2-13}, \eqref{eqn:R2-12}, and \eqref{eqn:R2-22} in \eqref{eqn:R2-bound2}, we get
	\begin{align*}
		\REGRET_2(T) & \leq 4 n \bs{L}_v \mathsf{D}_{\max} + 12 \bs{L}_v \sqrt{\dfrac{\bar{\mathsf{S}}_n T \log (CT)}{c}} \\
		& \leq 12 \bar{\mathsf{S}}_n \bs{L}_v \mathsf{D}_{\max} + 12 \bar{\mathsf{S}}_n \bs{L}_v \sqrt{\dfrac{T \log (CT)}{c}} \\
		& \leq 12 \bar{\mathsf{S}}_n \bs{L}_v \mathsf{D}_{\max} \sqrt{\dfrac{T\log (CT)}{c}}.
	\end{align*}
	
	\section{Whittle index computation} \label{app:whittle}
	
	Next, we describe a general method to compute the Whittle indices for any indexable RB. This methods is based on adaptive greedy algorithm proposed in \cite{akbarzadeh2022conditions}. Since we carry out the computations for each arm separately, with an abuse of notation, we drop the superscript~$i$ from all variables.
	
	Given an arm, a Markov policy $\pi \in \Pi$ for the arm, let $P^{(\pi)}$ denote the Markov chain matrix under policy~$\pi$, i.e., for any $s, y \in \mathcal{S}$, $P^{(\pi)}_{s, y} \coloneqq P_{s, y}(\pi(s))$. Also, let $r^{(\pi)}$ denote the per-step reward obtained under policy~$\pi$, i.e., for any $s \in \mathcal{S}$, $r^{(\pi)}(s) \coloneqq r(s, \pi(s))$. Let $(J^{(\pi)}_D, D^{(\pi)}(\cdot))$ and $(J^{(\pi)}_N, N^{(\pi)}(\cdot))$ be the solutions of the following equations: 
	\begin{align} 
		J^\pi_D + D^{(\pi)}(s) & \coloneqq r^{(\pi)}(s) + \langle P^{(\pi)}_s, D^{(\pi)} \rangle,
		\quad \forall s \in \mathsf{S} \label{eqn:D_comp} \\
		J^\pi_N + N^{(\pi)}(s) & \coloneqq \pi(s) + \langle P^{(\pi)}_s, \pi \rangle, 
		\qquad \forall s \in \mathsf{S}  \label{eqn:N_comp}
	\end{align}
	where $D^{(\pi)}(s)$ denotes the differential expected average reward and $N^{(\pi)}(s)$ denotes the differential expected average number of times active action is selected under policy~$\pi$ starting from the initial state~$s$. According to \cite{puterman2014markov}, these set of equations can be solved by setting a component of $D^{\pi}(\cdot)$ and $N^{\pi}(\cdot)$ to zero and solve a system of linear equations with $\mathsf{S}$ number of equations.
	
	For any subset $\mathcal{X} \subseteq \mathcal{S}$, define the policy $\bar{\pi}^{(\mathcal{X})}: \mathcal{S} \to \{0, 1\}$ as 
	\begin{equation*}
		\bar{\pi}^{(\mathcal{X})}(s) = 
		\begin{cases}
			0, ~ \text{ if } s \in \mathcal{X} \\
			1, ~ \text{ if } s \in \mathcal{S}\backslash\mathcal{X}.
		\end{cases}
	\end{equation*}
	
	Then, in each iteration of the while loop, denote $\mathcal{W}$ as the set of states where the Whittle index is already obtained for, denote $\mathcal{Y}$ as the set of states we obtain the Whittle index for, and denote $\xi^*$ as the Whittle index of the states in the set~$\mathcal{Y}$. The required computation is carried out below.
	\begin{align} 
		\Lambda_{y} & = \{y \in \mathcal{S}: N^{(\bar{\pi}^{(\mathcal{P})})}(s) \neq N^{(\bar{\pi}^{(\mathcal{P} \cup \{y\})})}(s) \}, \forall y \in \mathcal{S}\setminus\mathcal{W}, \label{eqn:Lambda} \\
		\xi_{y}(s) & = \frac{(J^{(\bar{\pi}^{(\mathcal{P})})}_D+D^{(\bar{\pi}^{(\mathcal{P})})}(s)) - (J^{(\bar{\pi}^{(\mathcal{P} \cup \{y\})})}_D + D^{(\bar{\pi}^{(\mathcal{P} \cup \{y\})})}(s))}{(J^{(\bar{\pi}^{(\mathcal{P})})}_N+N^{(\bar{\pi}^{(\mathcal{P})})}(s)) - (J^{(\bar{\pi}^{(\mathcal{P} \cup \{y\})})}_N + N^{(\bar{\pi}^{(\mathcal{P} \cup \{y\})})}(s))}, ~ \forall s \in \Lambda_{y}, \label{eqn:mu_y} \\
		\xi^* & = \min_{y \in \mathcal{S}\setminus\mathcal{W}} \min_{s \in \Lambda_{y}} \xi_{y}(s), \label{eqn:index} \\
		\mathcal{Y} & = \argmin_{y \in \mathcal{S}\setminus\mathcal{W}} \min_{s \in \Lambda_{y}} \xi_{y}(s) \label{eqn:setindex}.
	\end{align}
	The algorithm is presented in Alg.~\ref{Alg:Whittle index policy}.
	\begin{algorithm}[!t]
		\begin{algorithmic}[1]
			\STATE \textbf{Input: } Arm~$(\mathcal{S}, \{0, 1\}, P, r, I)$.
			\STATE $\mathcal{W} = \emptyset$.
			\WHILE{$\mathcal{W} \neq \mathcal{S}$}
			\STATE Compute $\Lambda_{y}$, $\xi_{y}(s)$, $\forall y \in \mathcal{S}\setminus\mathcal{W}$ and $\forall s \in \Lambda_{y}$ according to \eqref{eqn:Lambda} and \eqref{eqn:mu_y}.
			\STATE Get $\xi^*$ based on \eqref{eqn:index} and $\mathcal{Y}$ based on \eqref{eqn:setindex}.
			\STATE $w(y) \leftarrow \xi^*$, $\forall y \in \mathcal{Y}$
			\STATE $\mathcal{W} \leftarrow \mathcal{W} \cup \mathcal{Y}$ 
			\ENDWHILE
		\end{algorithmic} 
		\caption{Whittle Index Computation}
		\label{Alg:Whittle index policy}
	\end{algorithm}
	
	\section{Randomly generated model}\label{app:matrix}
	This model is same as the one presented in \cite[Appendix]{akbarzadeh2019restless}. Consider a Markov chain with $\mathsf{S}$ states and the transition probability matrix
	\begin{align*}
		P = 
		\begin{bmatrix}
			P_{11} & P_{12} & P_{13} & \dots  & P_{1\mathsf{S}} \\
			P_{21} & P_{22} & P_{23} & \dots  & P_{2\mathsf{S}} \\
			\vdots & \vdots & \vdots & \ddots & \vdots \\
			P_{\mathsf{S}1} & P_{\mathsf{S}2} & P_{\mathsf{S}3} & \dots  & P_{\mathsf{S}\mathsf{S}}
		\end{bmatrix}.
	\end{align*}
	Let $F_{ij} = \sum_{y = j}^{\mathsf{S}} P_{iy}$.
	The necessary condition for $P$ to be stochastic monotone is that for any
	$1 \leq i \leq l \leq \mathsf{S}$ and any $1 \leq j \leq \mathsf{S}$, $F_{ij} \leq F_{lj}$.
	
	Initially, we generate $P_{11}$ uniformly random between $[{1-d}, 1]$ where $d
	\in [0, 1]$. Variable~$d$ prevents the kernel to behave badly when the number
	of states increases. Then, we generate $P_{12}, P_{13}, \dots, P_{1\mathsf{S}}$
	sequentially from $P_{12}$ to $P_{1\mathsf{S}}$ where each mass is selected uniformly
	random from $[0,B_{i}]$ where $B_{i} = 1-\sum_{l = 1}^{i-1}P_{1l}$. As $F_{i\mathsf{S}}
	= P_{i\mathsf{S}}$ for any $i$, we select $P_{i\mathsf{S}}$ sequentially for rows from $2$ to
	$n$ where each element is generated uniformly random from $[P_{(i-1)n},
	\min\{1, P_{(i-1)\mathsf{S}}+d\}]$. Then, for any row from $2$ to $\mathsf{S}$, we repeat the
	following procedure backwardly for columns from $\mathsf{S}-1$ to $1$. Consider row~$i$
	and column~$j$. We generate a uniformly random number from $[\text{LB}_{ij},
	\text{UB}_{ij}]$ where $\text{LB}_{ij} = F_{(i-1)j} - F_{i(j+1)}$ and
	$\text{UB}_{ij} = \min\{\text{LB}+d, 1-F_{i(j+1)}\}$ and set the generated
	number as $P_{ij}$. The lower bound is due to stochastic monotonicity property
	and the upper bound is due to definition of a probability mass function and
	variable~$d$. Note that for the elements in the first column, the mentioned
	interval shrinks to $[1 - F_{i2}, 1 - F_{i2}]$ for row~$i$ which results in
	$P_{i1} = 1 - F_{i2}$.
	
	In Section~\ref{sec:numerical}, we set $d = 0.5/S$.
	
	%\bibliographystyle{cas-model2-names}

\end{document}